\documentclass[11pt]{article}

\usepackage{amsmath}
\usepackage{lipsum}
\usepackage{amsfonts}
\usepackage{graphicx}
\usepackage{epstopdf}
\usepackage{fullpage}
\usepackage{amsthm}
\usepackage{amssymb}
\usepackage{algorithm}
\usepackage{algpseudocode}
\usepackage{listings}
\usepackage{booktabs}
\usepackage{epsfig}
\usepackage{color}
\usepackage{float}
\usepackage{footnote}
\usepackage{mathrsfs}
\usepackage{multirow}
\usepackage{hyperref}
\usepackage{cleveref}

\usepackage{subfig}
\usepackage{tikz}
\usetikzlibrary{shapes}
\usetikzlibrary{positioning}
\usetikzlibrary{circuits}

\newcommand{\rowvec}[1]{\begin{bmatrix}#1\end{bmatrix}}

\newtheorem{remark}{Remark}[section]

\newtheorem{example}{Example}[section]
\newtheorem{lemma}{Lemma}[section]
\newtheorem{definition}{Definition}[section]

\newtheorem{theorem}{Theorem}[section]

\title{$\epsilon$-rank and the Staircase Phenomenon: New Insights into Neural Network Training Dynamics\thanks{Submitted to the editors DATE.
}}

\author{Jiang Yang\thanks{Department of Mathematics, SUSTech International Center for Mathematics \& National Center for Applied Mathematics Shenzhen (NCAMS), Southern University of Science and Technology, Shenzhen, China (\tt{yangj7@sustech.edu.cn}).}
	\and Yuxiang Zhao\thanks{Department of Mathematics, Southern University of Science and Technology, Shenzhen, China (\tt{12131241@mail.sustech.edu.cn}, \tt{12131244@mail.sustech.edu.cn}).}
	\and Quanhui Zhu\footnotemark[3]}

\usepackage{amsopn}

\ifpdf
\hypersetup{
  pdftitle={The Staircase Phenomenon in Training Dynamics of Neural Networks},
  pdfauthor={J. Yang, Q. Zhu, and Y. Zhao}
}
\fi

\begin{document}

\maketitle

\begin{abstract}
Understanding the training dynamics of deep neural networks (DNNs), particularly how they evolve low-dimensional features from high-dimensional data, remains a central challenge in deep learning theory. In this work, we introduce the concept of $\epsilon$-rank, a novel metric quantifying the effective feature of neuron functions in the terminal hidden layer. Through extensive experiments across diverse tasks, we observe a universal \textit{staircase phenomenon}: during training process implemented by the standard stochastic gradient descent methods, the decline of the loss function is accompanied by an increase in  the $\epsilon$-rank and exhibits a staircase pattern. Theoretically, we rigorously prove a negative correlation between the loss lower bound and $\epsilon$-rank, demonstrating that a high $\epsilon$-rank is essential for significant loss reduction. Moreover, numerical evidences show that within the same deep neural network, the $\epsilon$-rank of the subsequent hidden layer is higher than that of the previous hidden layer.  Based on these observations, to eliminate the staircase phenomenon, we propose a novel pre-training strategy on the initial hidden layer that  elevates the $\epsilon$-rank of the terminal hidden layer. Numerical experiments validate its effectiveness in reducing training time and improving accuracy across various tasks. Therefore, the newly introduced concept of $\epsilon$-rank is a computable quantity that serves as an intrinsic effective metric characteristic for deep neural networks, providing a novel perspective for understanding the training dynamics of neural networks and offering a theoretical foundation for designing efficient training strategies in practical applications.
\end{abstract}

\section{Introduction}
Deep neural networks (DNNs) have exhibited remarkable performance across a wide range of fields,
including computer vision, natural language processing, and computational physics,
due to their extraordinary representation capabilities.
A widely accepted perspective is that DNNs excel at capturing low-dimensional features embedded 
in high-dimensional data, a performance that underpins their success in tasks 
ranging from image recognition, reinforcement learning to solving partial differential equations.
While empirical evidence has established that neural networks excel at fitting low-dimensional 
structures in data, a theoretically sound tool to quantitatively analyze this process have been lacking.

The complexity of DNN training mechanisms fundamentally hinders the development of interpretable and generalizable models. This complexity arises from the interplay between the high-dimensional parameter space, non-convex optimization landscapes, and the dynamic evolution of functional representations during training. As a result, understanding how DNNs systematically extract and refine low-dimensional structures remains a critical challenge.

There are various approaches attempting to explain the mechanism of training dynamics in neural networks.
Visualization-based methods, such as those in \cite{zeilerVisualizingUnderstandingConvolutional2014, liVisualizingLossLandscape2018}, analyze the hierarchical formation of feature maps in convolutional networks and investigate how network architecture influences the geometry of the loss landscape during training. The impact of flat and sharp minima on generalization performance has been extensively studied in \cite{hochreiterFlatMinima1997, dinhSharpMinimaCan2017, DBLP:conf/iclr/KeskarMNST17}, providing insights into the relationship between the optimization landscape and generalization behavior.
The neural tangent kernel (NTK) theory provides a rigorous framework for analyzing the training dynamics of wide neural networks in the infinite-width limit. Jacot et al. \cite{jacotNeuralTangentKernel2018} introduce the NTK to describe the training dynamics as a convergent kernel in function space, offering valuable insights into how networks evolve during training. Extensions of the NTK framework in \cite{wangWhenWhyPINNs2022, biettiInductiveBiasNeural2019, chenGeneralizedNeuralTangent2020} further examine the generalization behavior of neural networks across various contexts.
Another important perspective comes from the frequency principle (or spectral bias), suggesting that neural networks tend to learn low-frequency patterns during the early stages of training \cite{xuFrequencyPrincipleFourier2020, rahamanSpectralBiasNeural2019, xuTrainingBehaviorDeep2019,xuOverviewFrequencyPrinciple2024}. 
This observation demonstrates that, when the training process of the neural network is projected into the spectral domain, the number of effective frequencies exhibits an increasing trend over time.

A multilayer perceptron (MLP) neural network can be constructed as 
\begin{equation}
	\left\{
	\begin{aligned}
		y_0 &= x, \\
		y_{k+1} &= \sigma(W_ky_{k}+b_k),& k = 0,\cdots,L-1,\\
		y & = \beta\cdot y_L,
	\end{aligned}\right.
	\label{st::NN}
\end{equation}
where $x\in\mathbb{R}^d$, $y_k\in \mathbb{R}^n$, and $W_0\in \mathbb{R}^{n\times d}, W_k\in \mathbb{R}^{n\times n}, b_k,\beta\in\mathbb{R}^n$ are trainable parameters. For simplicity, the widths of the hidden layers are chosen to be equal.
Denoting $\displaystyle \theta = \{W_k,b_k\}_{k=0}^{L-1}$, the output of the neural network can alternatively be expressed in the form 
\begin{equation*}
	y (x;\theta)= \sum_{j=1}^n\beta_j \phi_j(x;\theta).
\end{equation*}
We decompose the neural network into two parts: the \textbf{neuron functions} $\{\phi_j\}_{j=1}^n$ correspond to the neurons in the last hidden layer, and the coefficients $\{\beta_j\}_{j=1}^n$ are the weights in the output layer. 
This work focus on examining neural networks from the perspective of traditional computational mathematics.
Specifically, when the output of a neural network is expressed as a linear combination of the neurons in the last hidden layer, these neurons can be interpreted as a set of basis functions.
\cite{heReluDeepNeural2020,heReLUDeepNeural2022} show that ReLU-activated deep neural networks can reproduce 
all linear finite element functions and \cite{meethalFiniteElementMethodenhanced2023,mituschHybridFEMNNModels2021,ramuhalliFiniteElementNeuralNetworks2005}	
provide algorithms to combine classical finite element methods with neural networks. As discussed in \cite{huangUniversalApproximationUsing2006,nelsenRandomFeatureModel2021}, randomly generated neuron functions in shallow neural networks also exhibit sufficient representation ability. 
A detailed analysis of the coefficient matrices associated with the random features, particularly in terms of the distribution of their singular values, has been provided in \cite{chenOptimizationRandomFeature2024}.
The decay rate of the eigenvalues of the Gram matrix of a two-layer neural network with ReLU activation under general initialization has been analyzed in \cite{zhang2023shallownetworksstruggleapproximating}. Its further numerical study suggests that smoother activation functions lead to faster spectral decay of the Gram matrix. 
These findings underscore the connection between traditional computational methods and deep learning frameworks, motivating a deeper mathematical exploration of neural network representations.

In this work, we focus on understanding how neuron functions evolve during the training dynamics of neural networks. 
Motivated by prior researches, we investigate the behavior of the singular values of the Gram matrix associated with the neuron functions during the training process. To formalize the observations, we introduce the concept of \textbf{$\epsilon$-rank} for a set of functions (see as \cref{def::e_rank}), which quantitatively represents the number of effective features in the network.

Our key finding in this paper is the identification of a novel phenomenon concerning the $\epsilon$-rank of neuron functions, stated as follows:

\textit{\textbf{Staircase phenomenon}: In training dynamics, the loss function often decreases rapidly along with a significant growth of $\epsilon$-rank of neurons, and the evolution of the $\epsilon$-rank over time resembles a staircase-like pattern.}

\begin{figure}[htbp]
	\centering
	\includegraphics[width=0.8\textwidth]{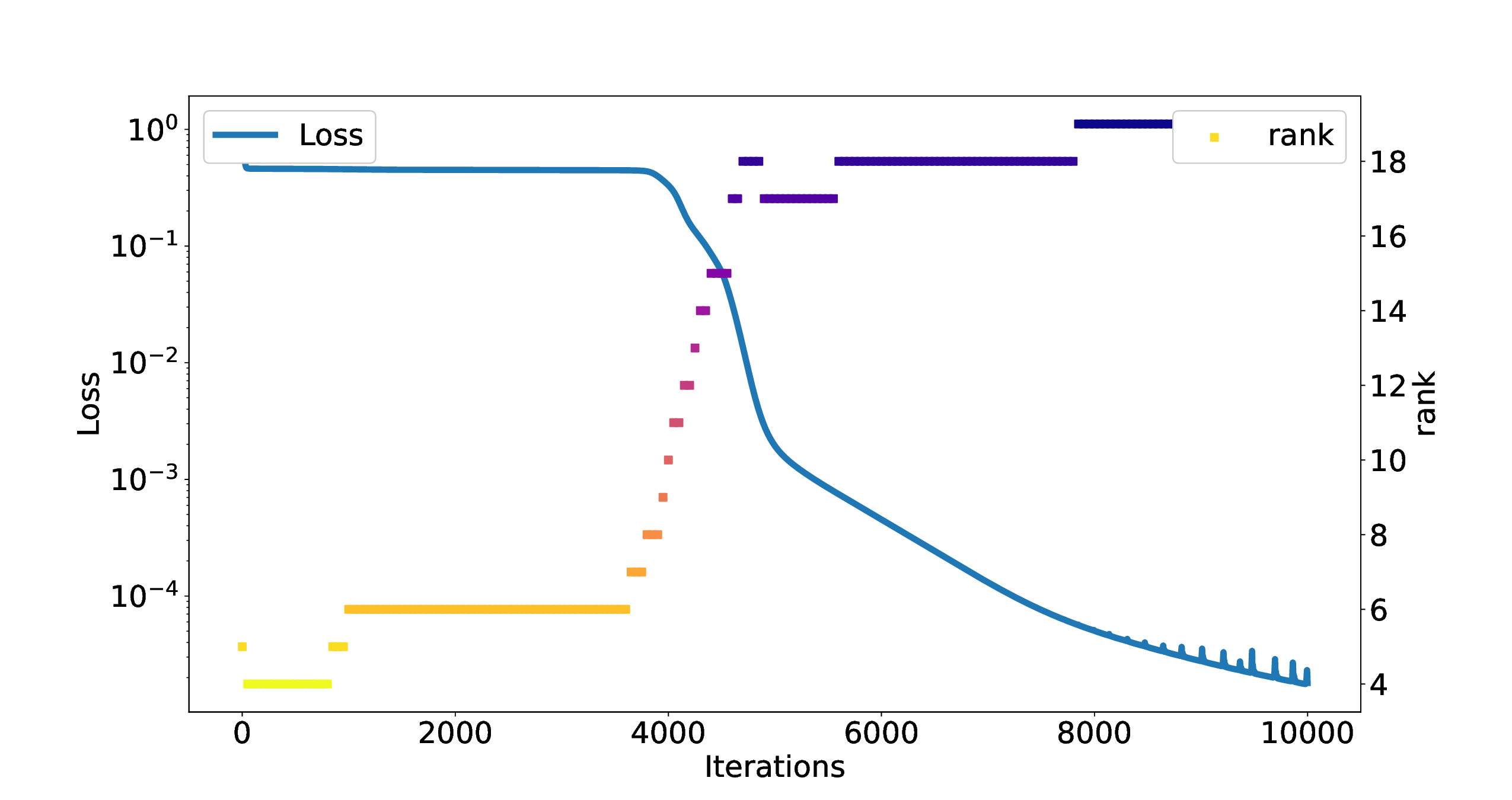}
	\caption{Staircase phenomenon of neuron functions in training dynamics.}
	\label{fig::toy example}
\end{figure}
As shown in the \cref{fig::toy example}, this increase in $\epsilon$-rank occurs in a stepwise fashion, closely resembling the structure of a staircase.
Specifically, under standard parameter initialization, neuron functions initially exhibit low linear independence.
Throughout the training process, variations in the loss function are closely linked to changes in the linear independence of these neuron functions. When the decline in the loss function reaches a plateau, the growth in linear independence tends to stabilize. Conversely, when the linear independence increases significantly, the loss experiences a rapid descent. We theoretically prove that the lower bound of the loss function decreases as the $\epsilon$-rank increases. This theoretical result holds for general deep neural networks.
Consequently, achieving a sufficiently large $\epsilon$-rank is essential for ensuring a significant reduction in the loss function.

This naturally raises the question: how can a set of highly linearly independent neuron functions be learned efficiently?
To address this, we propose a novel pre-training strategy that enable the neuron functions with a high $\epsilon$-rank before training. Specifically, by designing appropriate initialization schemes and selecting well-suited network architectures, it is possible to ensure that the neuron functions exhibit high linear independence from the outset. Such approaches significantly accelerate the training process or improve the accuracy.

\par

The main contributions of this work are summarized as follows:
\begin{itemize}
	\item[1.] \textbf{Identification of a universal training phenomenon}
	
	 The introduction of $\epsilon$-rank provides a novel perspective for understanding the training dynamics of neural networks, revealing a staircase phenomenon. This phenomenon is observed universally across various tasks, including function fitting, handwriting recognition, and solving partial differential equations. 
			
	\item[2.] \textbf{Theoretical characterization via $\epsilon$-rank dependent bounds}
	
	We prove that the loss function of deep neural networks has a lower bound related to the $\epsilon$-rank. 
	This bound decreases as the $\epsilon$-rank increases, thereby providing a theoretical explanation for the staircase phenomenon. 
	Our analysis and numerical examples conclude that a sufficiently high $\epsilon$-rank is a necessary condition for achieving a significant reduction in the loss function.
	
	\item[3.] \textbf{Pre-training strategy for $\epsilon$-rank enhancement}

	We propose a novel pre-training strategy for tanh-activated neural networks by constructing a set 
	of $\epsilon$-linearly independent neuron functions in the first hidden layer before training. This strategy enables the network to achieve a high $\epsilon$-rank at an early training stage, significantly reducing training time and improving model accuracy.
\end{itemize}

The rest of this paper is organized as follows. 
\cref{sec::ladder_phenomenon} introduces the concept of $\epsilon$-rank, and demonstrates the staircase phenomenon across diverse tasks and network configurations.
A lower bound on the loss function with respect to the $\epsilon$-rank is given in \cref{sec::theorem}, which provides the theoretical explanation of the staircase phenomenon. 
The particularly pre-training strategy is proposed in \cref{sec::initialization}, and numerical experiments validate its effectiveness in accelerating training dynamics and improving accuracy. 
Concluding remarks are given in the last section.

\section{Staircase Phenomenon}\label{sec::ladder_phenomenon}
\subsection{Preliminary}
Before presenting the staircase phenomenon, we introduce several definitions that are extensions of linear algebra.
\begin{definition}
	The $n$ functions $f_1(x),\cdots,f_n(x)$ are \textbf{linearly dependent} in domain $\Omega$ if, there exists $c_1,c_2,\cdots,c_n\in \mathbb{R}$ not all zero s.t. 
	\begin{equation}
		c_1f_1(x)+c_2f_2(x)+\cdots + c_nf_n(x)= 0,\qquad \forall x \in \Omega.
		\label{def::ld}
	\end{equation}
\end{definition}
If the functions are not linearly dependent, they are said to be \textbf{linearly independent}. 
For given $n$ functions $f_1(x),\cdots,f_n(x)$ in $L^2(\Omega)$, the Gram matrix $M$ is defined as follows:
\begin{equation}
	(M)_{ij} := \int_\Omega f_i(x)f_j(x)\mathrm{d}x,\quad 1\leq i,j\leq n.
\end{equation}
Obviously $M$ is symmetric positive semi-definite.

\begin{lemma}
	$f_1(x),\cdots,f_n(x)$ in are linearly independent if and only if the rank of the Gram matrix $r(M) = n$. 
\end{lemma}

While the concept of linear dependence is well-defined in linear algebra, few functions in real-world applications strictly satisfy the condition specified in \cref{def::ld}. Consequently, a more practical and applicable indicator, extending beyond the framework of \cref{def::ld}, is required.
\begin{definition}
	The $n$ functions $f_1(x),\cdots,f_n(x)$ in are \textbf{$\epsilon$-linearly dependent} in domain $\Omega$ for some $\epsilon \geq 0$ if, there exists $c_1,c_2,\cdots,c_n\in \mathbb{R}$ with $\|\mathbf{c}\|^2=1$ s.t.,
	\begin{equation}
		\|c_1f_1+c_2f_2+\cdots+c_nf_n\|^2 \leq \epsilon.
		\label{def::2d}
	\end{equation}
	Otherwise they are said to be \textbf{$\epsilon$-linearly independent}.
\end{definition}

Here $\|\cdot\|$ is the short note of $L^2$ norm, $\displaystyle \|u\|= \sqrt{\langle u,u\rangle} = \left(\int_\Omega u^2\mathrm{d}x\right)^\frac{1}{2}$. It is easy to check that $f_1(x),\cdots,f_n(x)$ are $\epsilon$-linearly independent if and only if the minimum eigenvalue of $M$ satisfies $\lambda_{\min}(M) > \epsilon.$ In the following, we intend to extend these concepts of linear independence to the neuron functions of a neural network.

\begin{definition}
For a given neural network $u(x;\theta)$ defined on $\Omega\subset \mathbb{R}^d$ as
$$u(x;\theta)=\sum_{j=1}^n\beta_j\phi_j(x;\theta),$$
  the Gram matrix $M_u$ is defined as below
  $$(M_u)_{ij}=\int_\Omega\phi_i(x;\theta)\phi_j(x;\theta)\mathrm{d}x.$$
\end{definition}

Straightforwardly, we have the following definition of $\epsilon$-rank.  

\begin{definition}
	\label{def::e_rank}
	The \textbf{$\epsilon$-rank} of a neural network $u(x,\theta)$, associated with its Gram matrix $M_u$, is defined as the number of eigenvalues exceeding a tolerance $\epsilon$, i.e.,
	\begin{equation*}
		r_\epsilon(M_u):= \#\{\lambda(M_u) > \epsilon\},
	\end{equation*}
	where $\lambda(M_u)$ are eigenvalues of $M_u$, and compute the algebraic number of repetitions.
\end{definition}

The standard definitions of linear dependence and rank can be regarded as special cases corresponding to $\epsilon = 0$. For simplicity, the linear independence and rank mentioned in experiments refer to $\epsilon$-linear independence and $\epsilon$-rank, respectively. 
The $\epsilon$-rank serves as an effective metric for observing the evolution of effective features in neural networks during training. It provides critical insights into the training dynamics, enabling the design of efficient training strategies based on the correlation between feature diversity and loss reduction. However, in practical problem-solving scenarios, 
explicit computation of the $\epsilon$-rank is unnecessary, as its theoretical role primarily lies in understanding the training 
mechanism of neural networks.

\subsection{Staircase Phenomenon}
With the necessary groundwork established, we now focus on how the $\epsilon$-rank of the neuron basis functions evolves throughout the training process of the neural network. Across various tasks, including function fitting, handwriting recognition, and solving partial differential equations, we consistently observe the following phenomenon.

\textit{\textbf{Staircase phenomenon}: In training dynamics, the loss function often decreases rapidly along with a significant growth of $\epsilon$-rank of neurons, and the evolution of the $\epsilon$-rank over time resembles a staircase-like pattern.} 
The $\epsilon$-rank, as defined in \cref{def::e_rank}, quantifies the linear independence or effective features of neuron functions in the last hidden layer. We experimentally demonstrate this phenomenon under different settings. To illustrate this phenomenon, we first consider a function fitting problem, which is one of the most fundamental tasks in neural networks.
\begin{example}[Stairecase phenomenon in function fitting]\label{eg::ff}
	Consider the target function composed of multiple frequency components, defined as \begin{equation}
		f(x) = \cos x + \cos 2x + \cos 30x.
		\label{eg::target}
	\end{equation}
	The computational domain is $\Omega = [-1,1]$, and the mean square error is used as the loss function. 
	
This example consists of three distinct experimental settings:
	\begin{itemize}
	\item[(i)] (\cref{fig::base_seperate}) Investigate the performance of neural networks with varying width and depth. The network configurations are as follows:
	 (a) $L=2, n=50$. (b) $L=2,n=25$. (c) $L=4,n=50$. (d) $L=4,n=25$.
	\item[(ii)]  (\cref{fig::activation}) Test neural networks with width of $n=50$, depth of $L=3$ and the following different activation functions:
	\begin{itemize}
		\item ReLU: $\sigma(x) = \max(x,0)$.
		\item ELU: $\sigma(x) = x,$ if $x>0$, and $\sigma(x)=\alpha(e^x-1)$, if $x<0$.
		\item Cosine: $\sigma(x) = \cos(x)$.
		\item Hyperbolic tangent: $\displaystyle \sigma(x) = \frac{e^x-e^{-x}}{e^x+e^{-x}}$
	\end{itemize}
	\item[(iii)]  (\cref{fig::depth}) 
	Analyze the $\epsilon$-rank of the neuron functions across different layers for a fixed network width of $n = 50$ and depth of $L = 4$.  
\end{itemize}
\end{example}

Under different width and depth, the training results are presented in \cref{fig::base_seperate}. 
\begin{figure}[htbp]
	\centering
	\subfloat[b][$L=2,n=50$]{\includegraphics[scale=0.38]{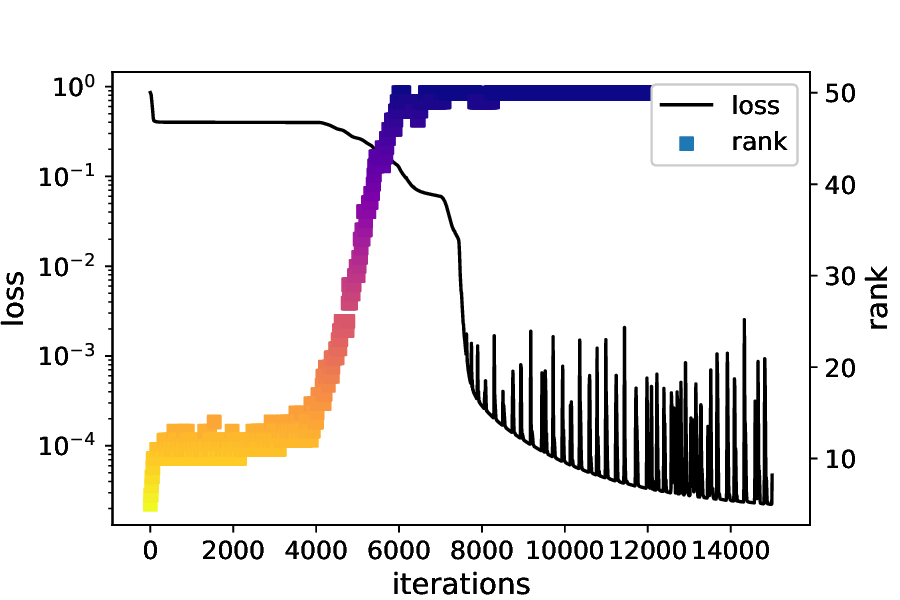}}
	\subfloat[b][$L=2,n=25$]{\includegraphics[scale=0.38]{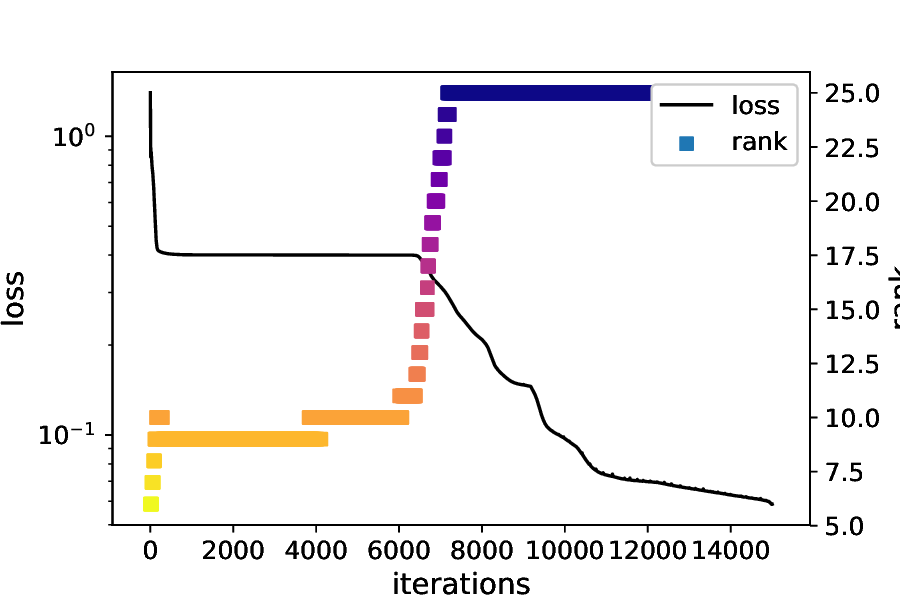}}
	
	\subfloat[b][$L=4,n=50$]{\includegraphics[scale=0.38]{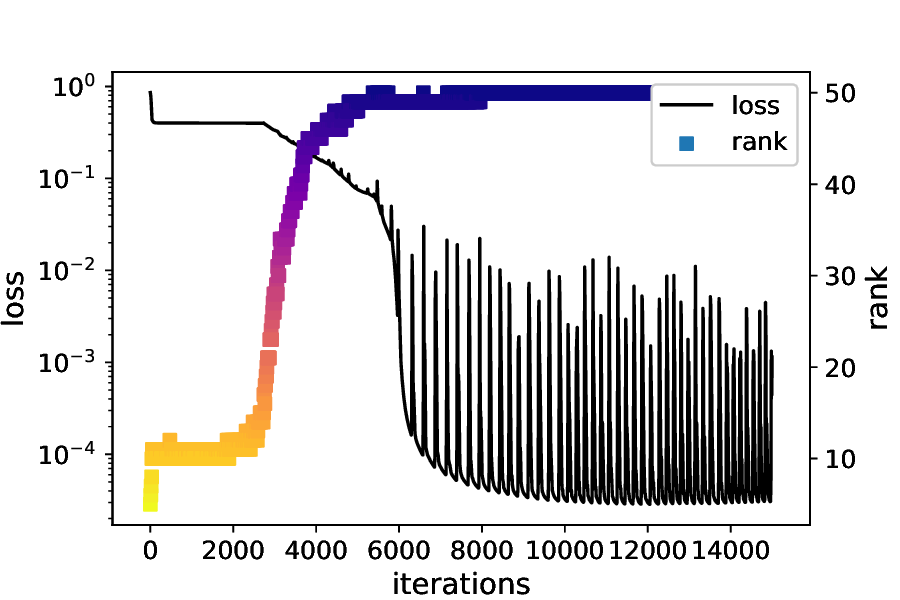}}
	\subfloat[b][$L=4,n=25$]{\includegraphics[scale=0.38]{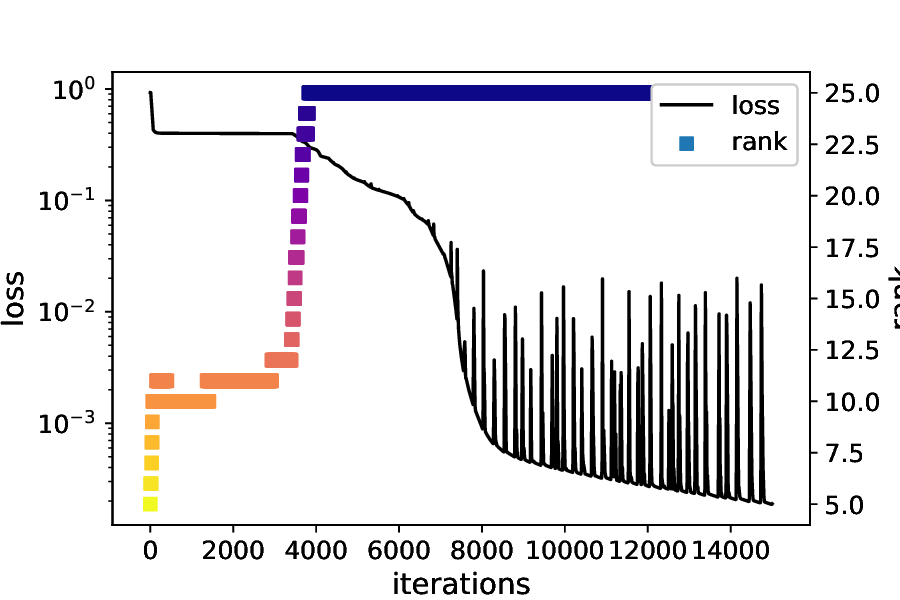}}
	\caption{(\cref{eg::ff}) Staircase phenomenon under varying network width ($n=25,50$) and depth ($L=2,4$). The black line represents the mean square error and the scatter plot indicates the $\epsilon$-rank.}
	\label{fig::base_seperate}
\end{figure}
The black line represents the loss function, while the scatters correspond to the $\epsilon$-rank of the Gram matrix. 
When the linear independence of the neuron functions remains unchanged, the loss function plateaus for thousands of iterations.
However, once the neuron functions become fully linearly independent, the loss function rapidly decreases to a lower level, as shown in \cref{fig::base_seperate}(a), (c) and (d). 
Subfigure (b) demonstrates that shallow and narrow networks fail to approximate the target function efficiently under insufficient training. 
Another notable observation from \cref{fig::base_seperate}(d) is that, after attaining full rank, the loss function experiences a further sharp decline.
This phase highlights a further unknown step in the learning process.

The appearance of staircase phenomenon is independent of the choice of activation function. 
\cref{fig::activation} illustrates the evolution of the $\epsilon$-rank during the training processes under several commonly used activation functions.
The results show that while its behavior differs depending on the specific activation function used, staircase phenomenon is present across these activation functions. It can be observed that the upward trend and the peak value of the $\epsilon$-rank vary across different activation functions. 
This is because neural network function classes constructed with different activation functions exhibit distinct approximation capabilities with respect to the target function.

\begin{figure}[htbp]
	\centering
	\subfloat[ReLU]{\includegraphics[width=0.48\textwidth]{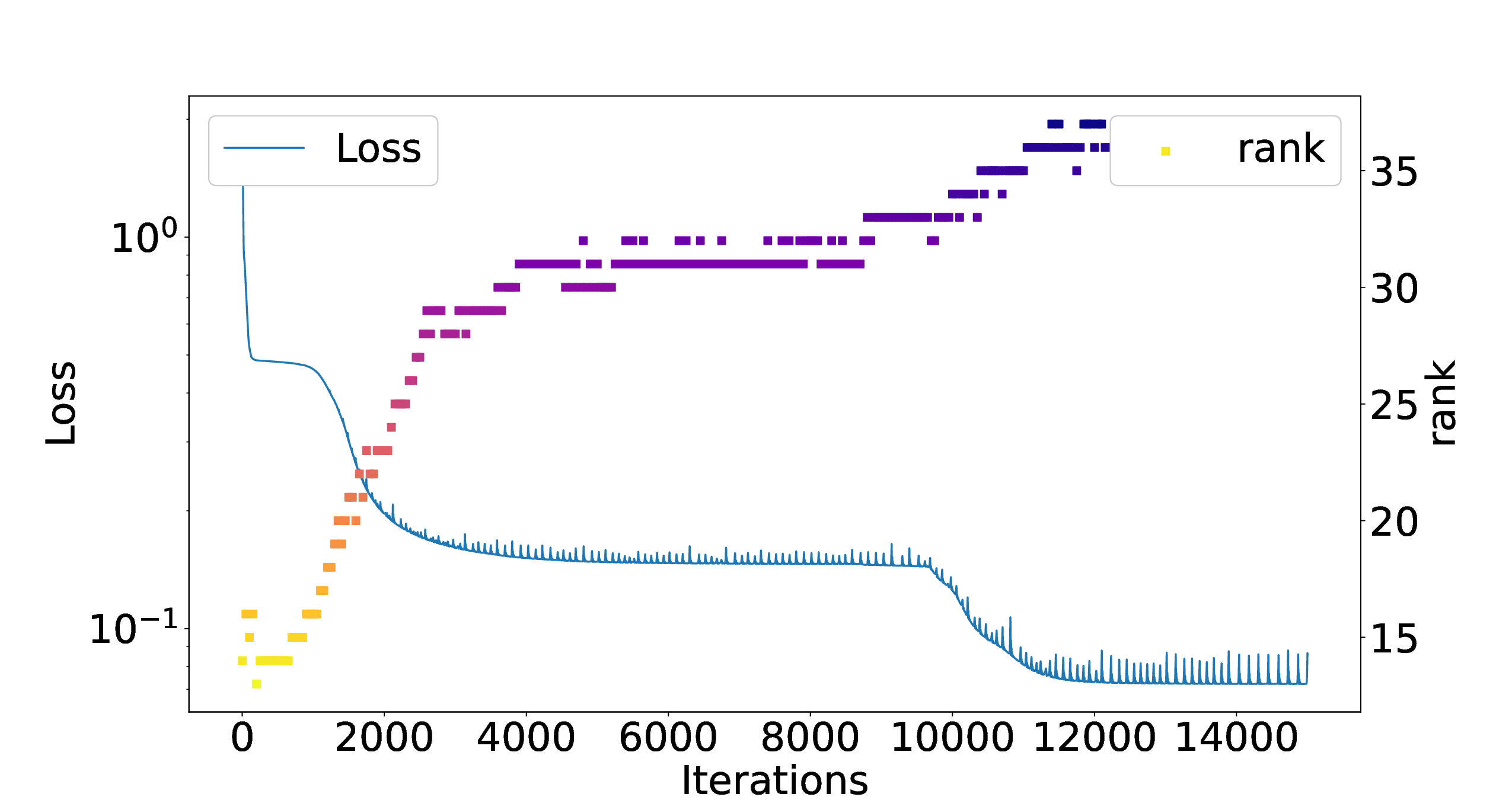}}
	\subfloat[ELU]{\includegraphics[width=0.48\textwidth]{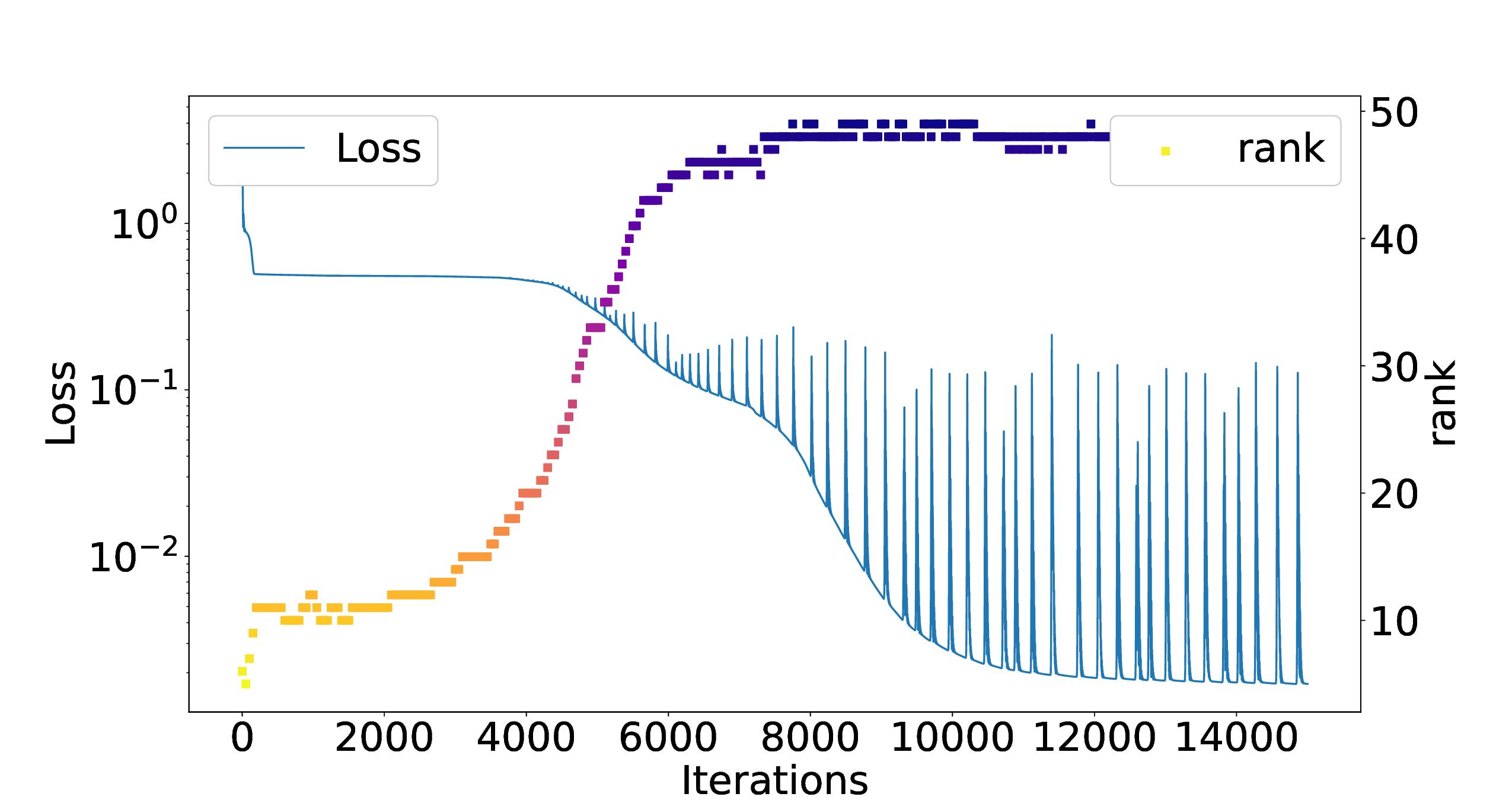}}

	\subfloat[Cosine]{\includegraphics[width=0.48\textwidth]{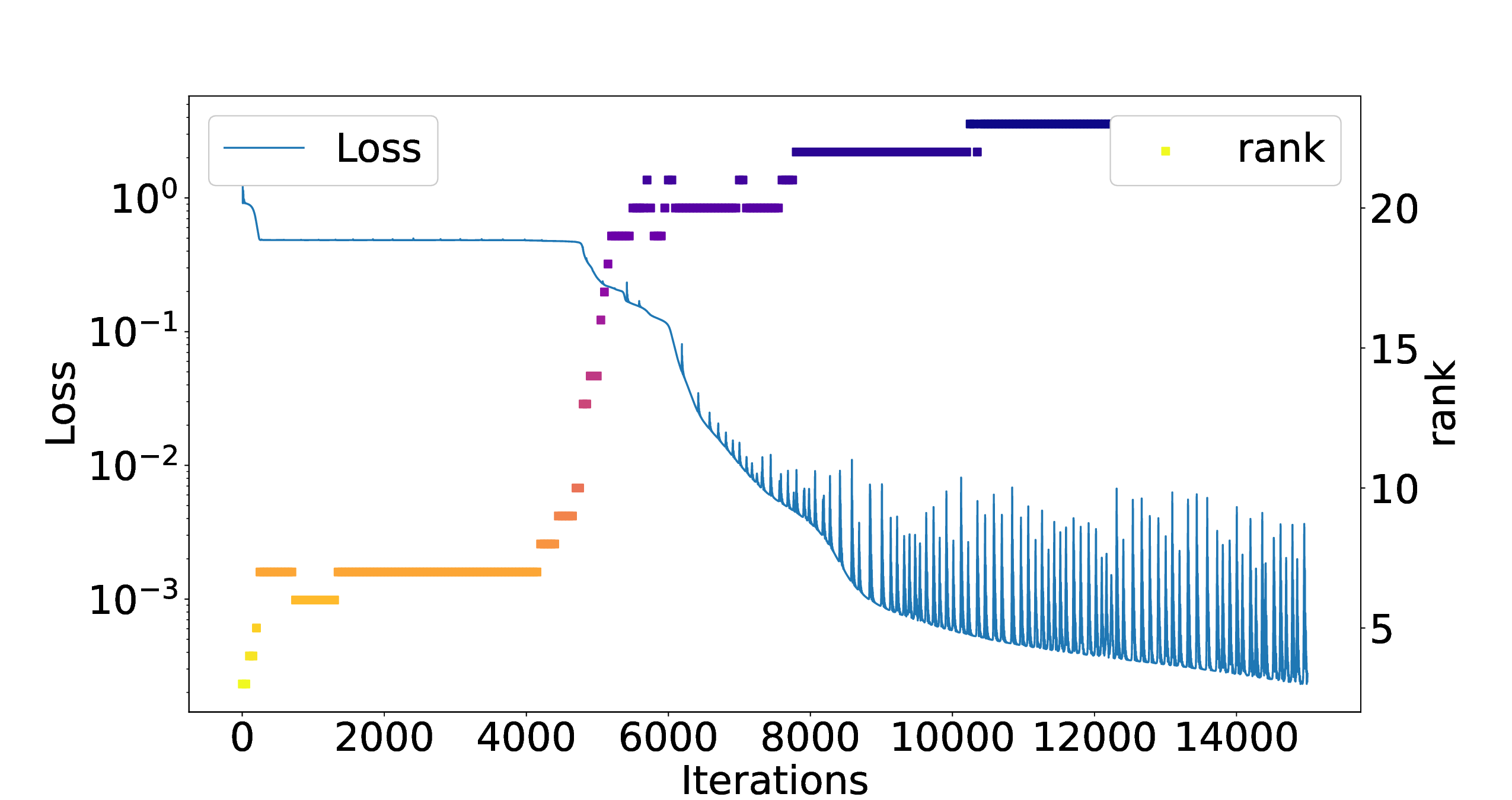}}
	\subfloat[Tanh]{\includegraphics[width=0.48\textwidth]{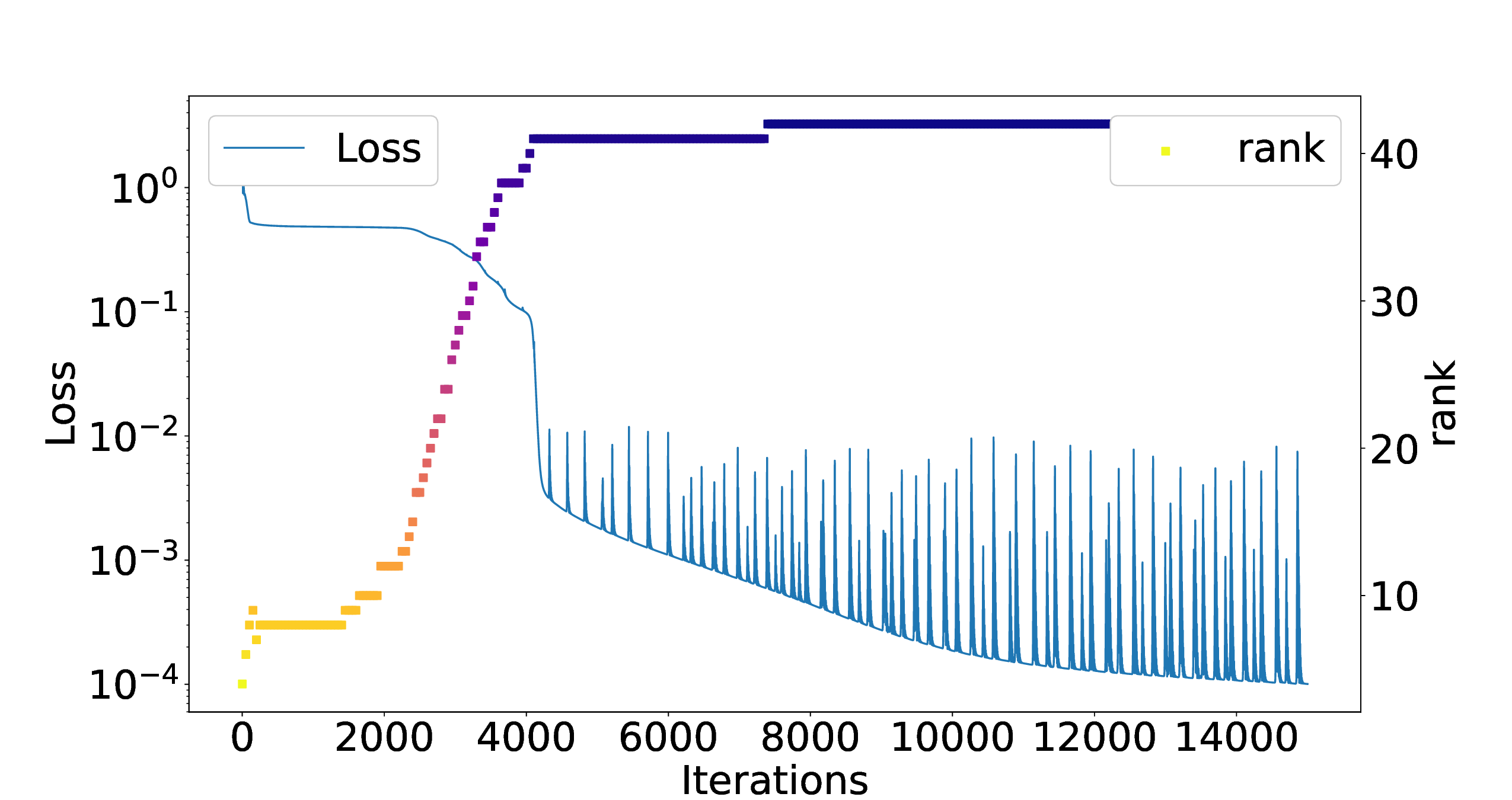}}

	\caption{(\cref{eg::ff}) Staircase phenomenon under different activation functions.}
	\label{fig::activation}
\end{figure}

We now demonstrate that the staircase phenomenon is not limited to function approximation but also occurs in solving partial differential equations (PDEs) and even in tasks like handwriting recognition. 
Since the pioneering work of \cite{raissiPhysicsinformedNeuralNetworks2019}, physics-informed neural networks (PINNs) have become a cornerstone in computational science for solving PDEs.
These methods have demonstrated robustness in high dimensional problems \cite{sirignanoDGMDeepLearning2018,eDeepRitzMethod2018,eDeepLearningbasedNumerical2017}, irregular domains \cite{zangWeakAdversarialNetworks2020,CiCP-28-2139}, and systems with complex dynamics \cite{liuMultiScaleDeepNeural2020,l.wightSolvingAllenCahnCahnHilliard2021}. However, as shown in \cref{fig::ac}, the staircase phenomenon manifests in PINN training: without principled initialization, the loss remains plateaued for thousands of iterations before decreasing. This issue extends beyond PDE solving—\cref{fig::handwriting} illustrates that similar training stagnation occurs in handwriting recognition tasks using standard initialization. These findings underscore the universality of the staircase phenomenon and the critical role of initialization design in overcoming training bottlenecks across diverse applications.

Moreover, this phenomenon not only occurs in function fitting, but also in solving partial differential equations (\cref{fig::ac}) and handwriting recognition (\cref{fig::handwriting}).

\begin{example}[Staircase phenomenon in solving Allen-Cahn equation,\cref{fig::ac}]\label{eg::ac}
	The Allen--Cahn equation is in the following form
	\begin{equation*}
			u_t = 0.0001 u_{xx} + u-u^3, \quad x\in [-1,1],\quad t\in[0,1].
	\end{equation*}
	The initial condition is $u_0(x) = \cos(\pi x)$ and periodic boundary condition is adapted. 
	The loss function is 
	\begin{equation*}
		\begin{aligned}
			\mathcal{L}(u) = &\|u_t - 0.0001 u_{xx} - u +u^3\|_{L^2([-1,1]\times [0,1])} + \\
			\mu_{ic}&\|u(\cdot, 0) - u_0\|_{L^2([-1,1])} + \mu_{bc}\|u(-1,\cdot) - u(1,\cdot)\|_{L^2([0,1])}^2.
		\end{aligned}
	\end{equation*}
\end{example}
\begin{example}[Staircase phenomenon in handwriting recognition, \cref{fig::handwriting}]\label{eg::hr}

	MNIST is a widely used database of handwritten digits for training various image processing systems. The dataset contains 70,000 grayscale images of handwritten digits (0 through 9), each with a resolution of 28x28 pixels. Thus, the dimension of the input is $d = 784$. In this example, the cross entropy loss function is employed for training.
\end{example}

\begin{figure}[htbp]
	\centering
	\includegraphics[width=0.7\textwidth]{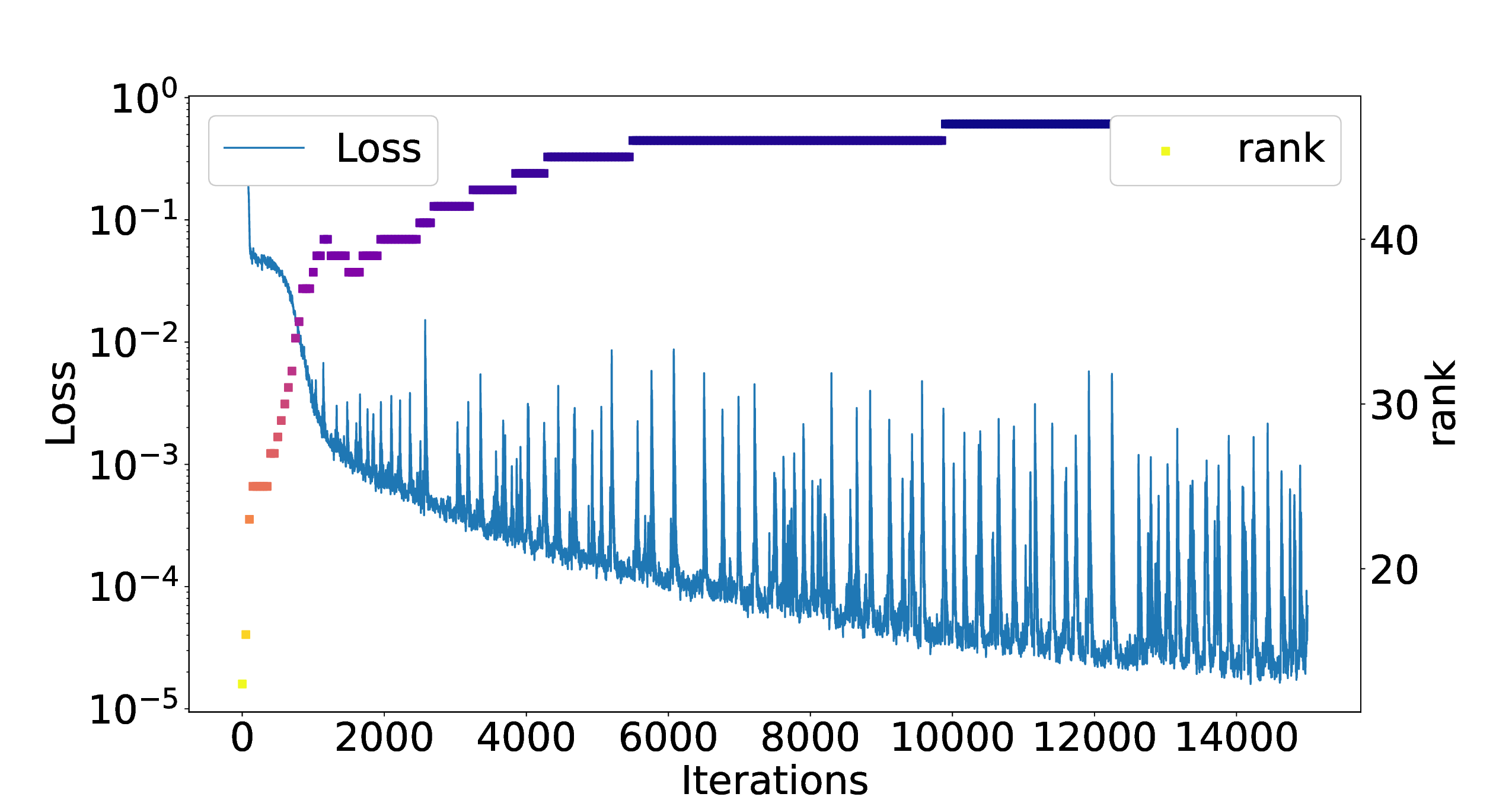}
	\caption{(\cref{eg::ac}) Staircase phenomenon in solving Allen--Cahn equation.}
	\label{fig::ac}
\end{figure}

\begin{figure}[htbp]
	\centering
	\includegraphics[width=0.7\textwidth]{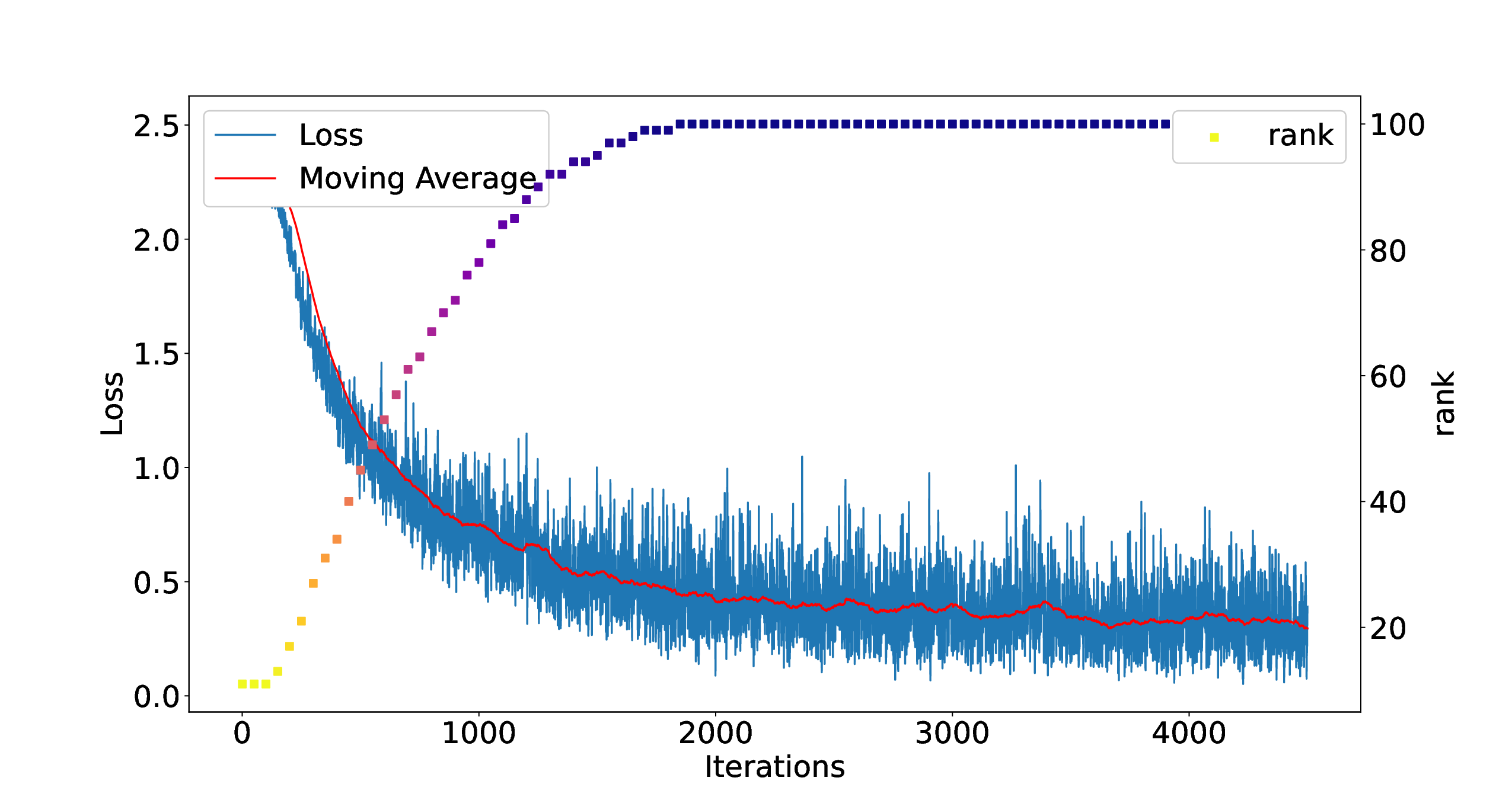}
	\caption{(\cref{eg::hr}) Staircase phenomenon in handwriting recognition.}
	\label{fig::handwriting}
\end{figure}

In \cref{fig::handwriting}, the staircase phenomenon is evident in the handwriting recognition task, demonstrating that this phenomenon can also be observed in high-dimensional problems.
This observation underscores the representational power of neural networks, which can effectively construct high-dimensional feature spaces, that is notably challenging in traditional scientific computing.

\section{Theoretical Explanation of the Staircase Phenomenon}\label{sec::theorem}
\subsection{Necessary Condition}
In this section, we demonstrate the theoretical significance of $\epsilon$-rank and provide an explanation of the staircase phenomenon. 
For well-posed problems, when the solution lies outside the class of neural network functions, we demonstrate that an increase in $\epsilon$-rank during the training process is a necessary condition for a reduction in the loss function.

It is trivial that if $f$ is a linear combination of $f_1,f_2,\cdots,f_n$ with $r(M) = p$, then $f$ can be rewritten as a linear combination of $p$ of them. Similarly, if $r_{\epsilon}(M) = p$, a comparable result holds.
We begins with the following useful lemma constructed in \cite{hongRankrevealingFactorizations1992}. 
\begin{lemma}\cite{hongRankrevealingFactorizations1992}\label{lem:orthogonal}
	Let $Q\in \mathcal{O}_{n,p}$ with $p\leq n$, and $\{Q_1,Q_2,\cdots,Q_k\}$ is the set of all $p$-by-$p$ sub-matrices of $Q$ where $k=\begin{pmatrix}
		n\\p
	\end{pmatrix}$. Define a vector in $\mathbb{R}^k$ by $\sigma(Q)=\left(\sigma_{\min}\left(Q_1\right),\sigma_{\min}\left(Q_2\right),\cdots,\sigma_{\min}\left(Q_k\right)\right)^T$. Then 
	$$\inf_{Q\in \mathcal{O}_{n,p}}\left(\|\sigma\left(Q\right)\|_{\infty}\right)\geq\frac{1}{\sqrt{p(n-p)+\min(p,n-p)}}.$$
\end{lemma}
\begin{theorem}\label{thm::rpineq}
	If $\displaystyle f = \sum_{j=1}^{n}\beta_j f_j $ with $\|\mathbf{\beta}\|^2 \leq C$, $r_{\epsilon}(M_f) = p$ for some $\epsilon \geq 0$ and positive integer $p \leq n$, then after a reorder of set $\{f_1,\cdots,f_n\} = \{\tilde{f_1},\cdots,\tilde{f_n}\}$ if necessary, $f$ can be approximated by $\displaystyle \tilde{f} = \sum_{j=1}^{p}\tilde{\beta_j} \tilde{f_j}$ with $\|f-\tilde{f}\|^2 \leq C(p+1)(n-p)^2\epsilon$.
\end{theorem}
\begin{proof}
	Consider the spectral decomposition $M = Q\Lambda Q^T$, $Q$ is orthogonal, $\Lambda = \text{diag}([\lambda_1,\lambda_2,\cdots,\lambda_n])$ is diagonal with eigenvalues $\lambda_1\geq\cdots\geq\lambda_p>\epsilon\geq\lambda_{p+1}\geq\cdots\geq\lambda_n\geq0$. For simplicity, we denote $\beta=\rowvec{\beta_1&\cdots&\beta_n}^T$ and $F=\rowvec{f_1&\cdots&f_n}^T$, thus $f=\beta^TF$. Let $P$ be a permutation matrix, and denote
	\[
	PQ = \left[ \begin{array}{cc}
		V_{11} & V_{12} \\
		V_{21} & V_{22}
	\end{array} \right]
	\begin{array}{l}
		\scriptstyle p \\
		\scriptstyle n - p
	\end{array}
	\vspace{-1em} 
	\]
	\[
	\begin{array}{cc}
		\scriptstyle p & \hspace{0.5em}\scriptstyle n-p
	\end{array}
	\]
	as a partitioning of the matrix $PQ$. Now we construct an approximation of $f$ by $\tilde{f}=\tilde{\beta}^T\tilde{F}$, where
	\begin{equation*}
		\begin{aligned}
			\tilde{\beta}&=\rowvec{I_p& 0}P\beta-V_{12}V_{22}^{-1}\rowvec{0&I_{n-p}}P\beta,\\
			\tilde{F}&=\rowvec{I_p& 0}PF.
		\end{aligned}
	\end{equation*}
	To estimate $L_2$ error of $\|f-\tilde{f}\|^2$, we note that
	$$
	\begin{aligned}
		&\quad\,\,\|f-\tilde{f}\|^2\\&=\|(P\beta)^T(PF)-\left(\rowvec{I_p& 0}P\beta-V_{12}V_{22}^{-1}\rowvec{0&I_{n-p}}P\beta\right)^T\rowvec{I_p& 0}PF\|^2\\
		&=\|\left(\rowvec{0&I_{n-p}}P\beta\right)^T\left(\rowvec{0&I_{n-p}}PF\right)+\left(\rowvec{0&I_{n-p}}P\beta\right)^TV_{22}^{-T}V_{12}^T\rowvec{I_p& 0}PF\|^2\\
		&=\|\left(\rowvec{0&I_{n-p}}P\beta\right)^TV_{22}^{-T}\left(V_{22}^T\rowvec{0&I_{n-p}}PF+V_{12}^T\rowvec{I_p& 0}PF\right)\|^2\\
		&=\|\left(\rowvec{0&I_{n-p}}P\beta\right)^TV_{22}^{-T}\left(PQ\rowvec{0&I_{n-p}}^T\right)^TPF\|^2\\
		&=\|\left(\rowvec{0&I_{n-p}}P\beta\right)^TV_{22}^{-T}\rowvec{0&I_{n-p}}Q^TF\|^2\\
		&\leq \|\beta\|^2\|V_{22}^{-T}\|^2\left(\lambda_{p+1}+\cdots+\lambda_n\right)\\
		&\leq C\|V_{22}^{-T}\|^2(n-p)\epsilon.
	\end{aligned}
	$$
	It remains to control the term $\|V_{22}^{-T}\|^2$. By \cref{lem:orthogonal}, for any orthogonal matrix $Q$, there exists a permutation $P$, such that $\|V_{22}^{-T}\|^2\leq (p+1)(n-p)$.
\end{proof}

It is worth noting that the bound established in the theorem is quite loose. The final inequality in the proof provides an upper bound
for $\lambda_{p+1}+\cdots+\lambda_n$ as $(n-p)\epsilon$, even though eigenvalues often decay rapidly in practice. Moreover, the bound in \cref{lem:orthogonal} is not sharp, except for $p=1$ and $p=n-1$. We conjecture that $\frac{1}{\sqrt{n}}$ provides the sharp bound for all $p$, supported by extensive numerical experiments that have yet to produced a counterexample.

We then apply the theorem to the neural networks. The $L$-layer neural network is represented as follows:
\begin{equation}
	\begin{aligned}
		\mathcal{F}_n &=\left\{\sum_{i=1}^n\beta_j\phi_j(x;\theta)\bigg|\phi_j(x;\theta)\in\mathcal{H}_L,\beta_j \in \mathbb{R}, j = 1,\cdots,n\right\},\\
		 \mathcal{H}_k &= \left\{\sigma(w_k\cdot y(x)+b_k)\bigg|y(x)\in\mathbb{R}^{n_{k-1}}, y_j(x)\in \mathcal{H}_{k-1}, \right.\\
		 &\qquad  \left. w_k \in\mathbb{R}^{n_{k-1}}, b_k\in\mathbb{R}\vphantom{\bigg|}\right\},\quad  k=1,\cdots,L,\\
		 \mathcal{H}_0 & = \left\{ x\in\Omega\subset\mathbb{R}^d\right\}.
	\end{aligned}
\end{equation}
where $\{\beta_j\}_{j=1}^n$ and $\displaystyle \theta = \{(W_k,B_k)\}_{k=1}^L$ are trainable coefficients, and $\sigma$ is the activation function satisfying the universal approximation theorem of neural networks \cite{hornikApproximationCapabilitiesMultilayer1991}. 
The two-layer (shallow) neural network is a special case when $L=1$.

Consider the following loss function:
\begin{equation}\label{lossl2ff}
	\min_{u\in \mathcal{F}_n}\mathcal{L}(u) = \|\mathcal{G}(u) - f\|^2,
\end{equation}
where  $f\in L^2(\Omega)$ represents the target function or data. In this work, we always assume that the exact solution $u^*$ does not belong to the function class $\mathcal{F}_n$. Obviously, with standard definition of the linear independence, the Gram matrix of a optimizer $u_n^*\in\mathcal{F}_n$ should be full rank. Otherwise, some redundant neurons in the last layer can be deleted. More concisely, $r(M_{u_n^*}) = n$. 
For a wider network, we can always find a better approximation, i.e., 
\begin{equation*}
	\text{dist}(u^*,\mathcal{F}_n)<\text{dist}(u^*,\mathcal{F}_m),\quad m<n,
\end{equation*}
where $\text{dist}(u,A) = \min_{v\in A}\|u-v\|$. However, the situation will become quite different if we consider the concept of the $\epsilon$-rank. We get a more precise relationship between the loss function and the $\epsilon$-rank. 

 \begin{theorem}\label{thm::rpineq2}
	Given the problem $\mathcal{G}(u) = f$ and $\mathcal{G}^{-1}$ being the solution operator, assume that the problem satisfies the following stability condition,
	for any $u,v$,
	\begin{equation}\label{assmp}
		\|\mathcal{G}^{-1}(u)-\mathcal{G}^{-1}(v)\|\leq C_S\|u-v\|.
	\end{equation}
	Let $u_n$ be an arbitrary approximation in $\mathcal{F}_n$  with the $\epsilon$-rank equalling to $p$, i.e., 

	$$\displaystyle u_n(x;\theta) = \sum_{j=1}^n\beta_j \phi_j(x;\theta),$$
	 where $\|\beta\|^2\leq C$ and $r_\epsilon(M_{u_n})=p$. 
	Then
	\begin{equation}
		\sqrt{\mathcal{L}(u_n)}\ge \frac{1}{C_s}\left(dist(u^*,\mathcal{F}_p)- \sqrt{C(p+1)(n-p)^2\epsilon}\right),
	\end{equation}
	where $u^* = \mathcal{G}^{-1}(f)$ is the exact solution, $\mathcal{L}(u) = \|\mathcal{G}(u)-f\|^2$ is the loss function.
\end{theorem}
\begin{proof}
	By \cref{thm::rpineq}, there exists a $u_p\in \mathcal{F}_p$ satisfying $\|u_p-u_n\|^2\leq C(p+1)(n-p)^2\epsilon$. 
	By the triangular inequality, 
	\begin{equation*}
		\begin{aligned}
			\text{dist}(u^*,\mathcal{F}_p)&\leq\|u^*-u_p\|\\
			&\leq\|u_p-u_n\|+\|u_n-u^*\|\\
			&\leq\sqrt{C(p+1)(n-p)^2\epsilon}+\|u_n-u^*\|\\
			&=\sqrt{C(p+1)(n-p)^2\epsilon}+\|\mathcal{G}^{-1}(\mathcal{G}(u_n))-\mathcal{G}^{-1}(f)\|\\
			&\leq \sqrt{C(p+1)(n-p)^2\epsilon}+C_S\|\mathcal{G}(u_n)-f\|.
		\end{aligned}
	\end{equation*}	
	gives the desired inequality.
\end{proof}
\begin{remark}
	The stability assumption \eqref{assmp} is reasonable for well-posed problems. For instance, in the function fitting problems, $\mathcal{G}=\mathcal{I}$, $C_S = 1$, and the loss function is  $\mathcal{L}(u) = \|u - f\|^2$. In the case using PINN to solve Poisson equation $-\Delta u=f$, $C_S = \|(-\Delta)^{-1}\|$ is uniformly bounded.
\end{remark}

\begin{remark}
When the $\epsilon$-rank of the neural network approximator $u_n\in\mathcal{F}_n$ is $r_\epsilon\left(M_{u_n}\right)=p$, the result of the theorem can be expressed as 
	\begin{equation*}
		\sqrt{\mathcal{L}(u_n)} \ge \frac{\text{dist}(u^*,\mathcal{F}_p)}{C_S}-O(\sqrt{\epsilon}).
	\end{equation*}
	Recall that $\epsilon$ is a fixed hyperparameter, and it is chosen sufficiently small, which yields the loss function has a lower bound in terms of $\text{dist}(u^*,\mathcal{F}_p)$.
	During the training process, to minimize the loss function $\mathcal{L}(u)$, there must be a decrease in $\text{dist}(u^*,\mathcal{F}_p)$, implying an increase in the $\epsilon$-rank of the neuron functions.
\end{remark}

\subsection{Numerical Validation}

We use two examples to intuitively explain the theorem \cref{thm::rpineq2}.
\begin{example}[Heat Equation, \cref{fig::fail}]\label{eg::heat}
	Consider the heat equation 
	\begin{equation*}
		u_t = 0.02 \Delta u,\quad \Omega = [-\pi,\pi]^2,
	\end{equation*}
	with $u_0(x,y) = \sin 5x \sin 5y$ and zero Dirichlet boundary condition. The analytic solution is $u(x,y,t) = e^{-t}\sin 5x\sin 5y$. 
	There are two different sampling settings: 
	\begin{itemize}
		\item[(i)] $n_1=1000,n_2=1000,n_3=50$;
		\item[(ii)]  $n_1=2500, n_2=10000, n_3=50$.
	\end{itemize}
	$n_1,n_2,n_3$ are the number of sampling points inside the domain, at the initial value and at the boundary, respectively. The network has a depth of $L=3$ and a width of $n=100$.
\end{example}
\begin{figure}[htbp]
	\centering
	\subfloat[failed settings]{\includegraphics[width=0.45\textwidth]{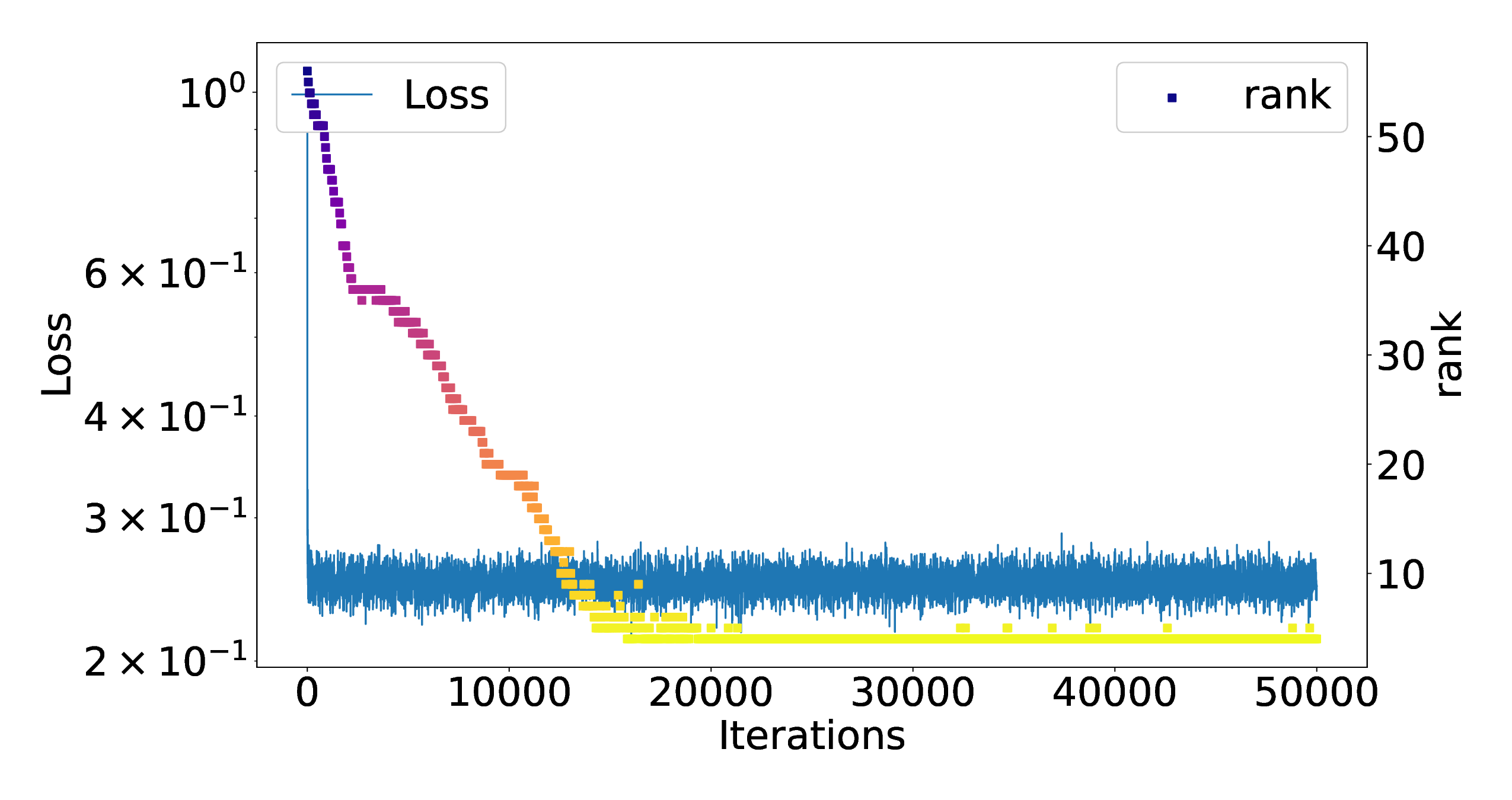}}
	\subfloat[trainable settings]{\includegraphics[width=0.45\textwidth]{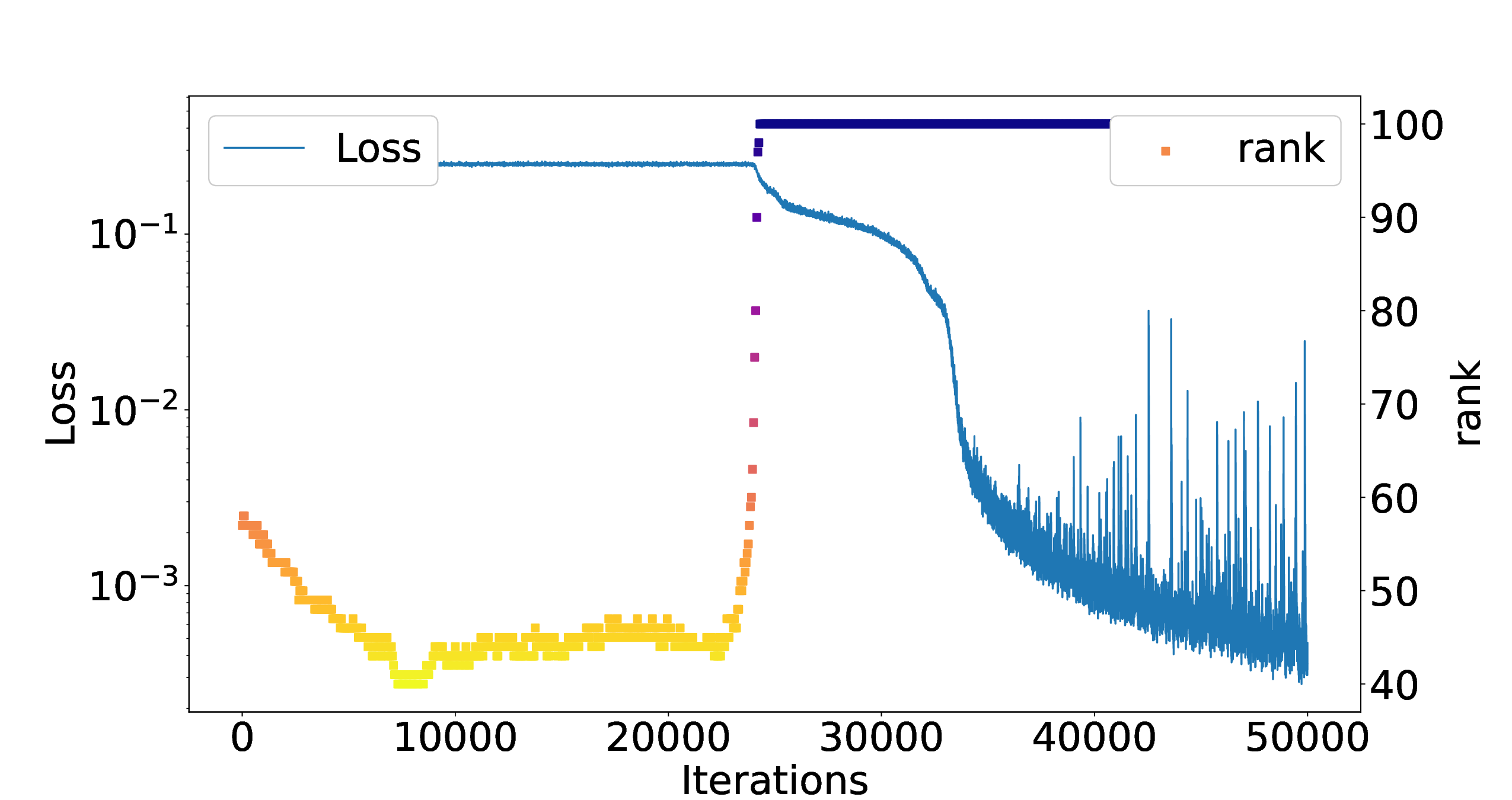}}
	\caption{(\cref{eg::heat}) The training process under different sampling settings. The left subfigure lacks sufficient number of samples on the domain and initial value while the right subfigure is trainable. The $\epsilon$-rank of the fail settings becomes low and the loss can not decrease. In the trainable setup, the staircase phenomenon occurred normally. }
	\label{fig::fail}
\end{figure}

In insufficient configurations, the $\epsilon$-rank becomes critically low, causing the loss function to remain at an elevated level and fail to decrease. When the $\epsilon$-rank $p$ is small, the approximation error $\text{dist}(u^*,\mathcal{F}_p)$ dominates the optimization landscape. Specifically, if the $\epsilon$-rank of the neuron functions does not decrease (i.e., the network fails to expand its representational capacity), the loss function $\mathcal{L}(u)$ remains trapped in a high-value plateau for an extended period. This phenomenon is typically observed during the initial stage or failure scenarios, where insufficient functional diversity limits the network's ability to approximate $u^*$. This explains why a sufficiently large $\epsilon$-rank, is a necessary condition for the loss function to decrease.

Next, we evaluate the $\epsilon$-rank in two methods that do not involve training process. Both extreme learning machine (ELM) \cite{huangExtremeLearningMachine2006b}  and random features \cite{rahimiRandomFeaturesLargescale2007} are widely used techniques in deep learning.
Notably, neither method involves training the hidden layers, offering computationally efficient solutions.
In these approaches, the linear independence of the neuron functions depends solely on the initialization and network structure employed. 

An important improvement is the partition of unity method (PoU) technique, employed in the random feature method (RFM) \cite{chenRandomFeatureMethod2023} and local extreme learning machine \cite{dongLocalExtremeLearning2021}. In RFM, the approximate solution is expressed as a linear combination of random features combined with the PoU, as follows:
\begin{equation*}
	u_R(x) = \sum_{i=1}^m \psi_i(x)\sum_{j=1}^{J_R} u_{ij}\phi_{ij}(x),
\end{equation*}
where $u_{ij}$ are unknown coefficients. In RFM, $N$ points $\{x_i\}_{i=1}^N$ are chosen from $\Omega$, typically uniformly distributed.
Then $\Omega$ is decomposed to $N$ disjoint subdomains $\{\Omega_i\}_{i=1}^N$ with $x_i\in\Omega_i$. For each $\Omega_i$, a PoU function $\psi_i$ is constructed with support $\Omega_i$, i.e., $\text{supp}(\psi_i) = \Omega_i$.
The extreme learning machine is modeled as 
\begin{equation*}
	u_E(x) = \sum_{j=1}^{J_E} u_j\phi_j(x),
\end{equation*}
which can be seen as a random feature method with no subdomains ($N=1$).

\begin{example}[Comparative $\epsilon$-rank study of RFM and ELM, \cref{fig::elm_rfm}]\label{eg::rfm}
	Consider a 2d approximation problem:
	\begin{equation*}
		u^*(x,y) = \cos x \cos y  + \cos 10x \cos 10y, \quad \Omega  = [-1,1]^2.
	\end{equation*}
In this example, we compare the performance of modeling with and without the partition of unity (PoU) technique, regarded as the RFM and ELM respectively.
The number of neurons is set to $n=900$. In the RFM, these $900$ neurons are divided into $3\times 3$ sub-intervals, i.e., $i = 9, J_R = 100$ and in the ELM, $J_E= 900$. 
The coefficients of the output layer in both methods are determined using the least squares method.
\end{example}

The results, presented in \cref{fig::elm_rfm}, demonstrate that, for the same number of neurons, the RFM achieves greater linear independence due to the compact support in each subdomain, resulting in higher accuracy. This example clearly illustrates that, under identical configurations, achieving a high linear independence through specific techniques can significantly enhance network performance. Furthermore, an approximate inverse proportional relationship between the $\epsilon$-rank and the error is observed. 

\begin{figure}[htbp]
	\subfloat[solution of ELM]{\includegraphics[width=0.33\textwidth]{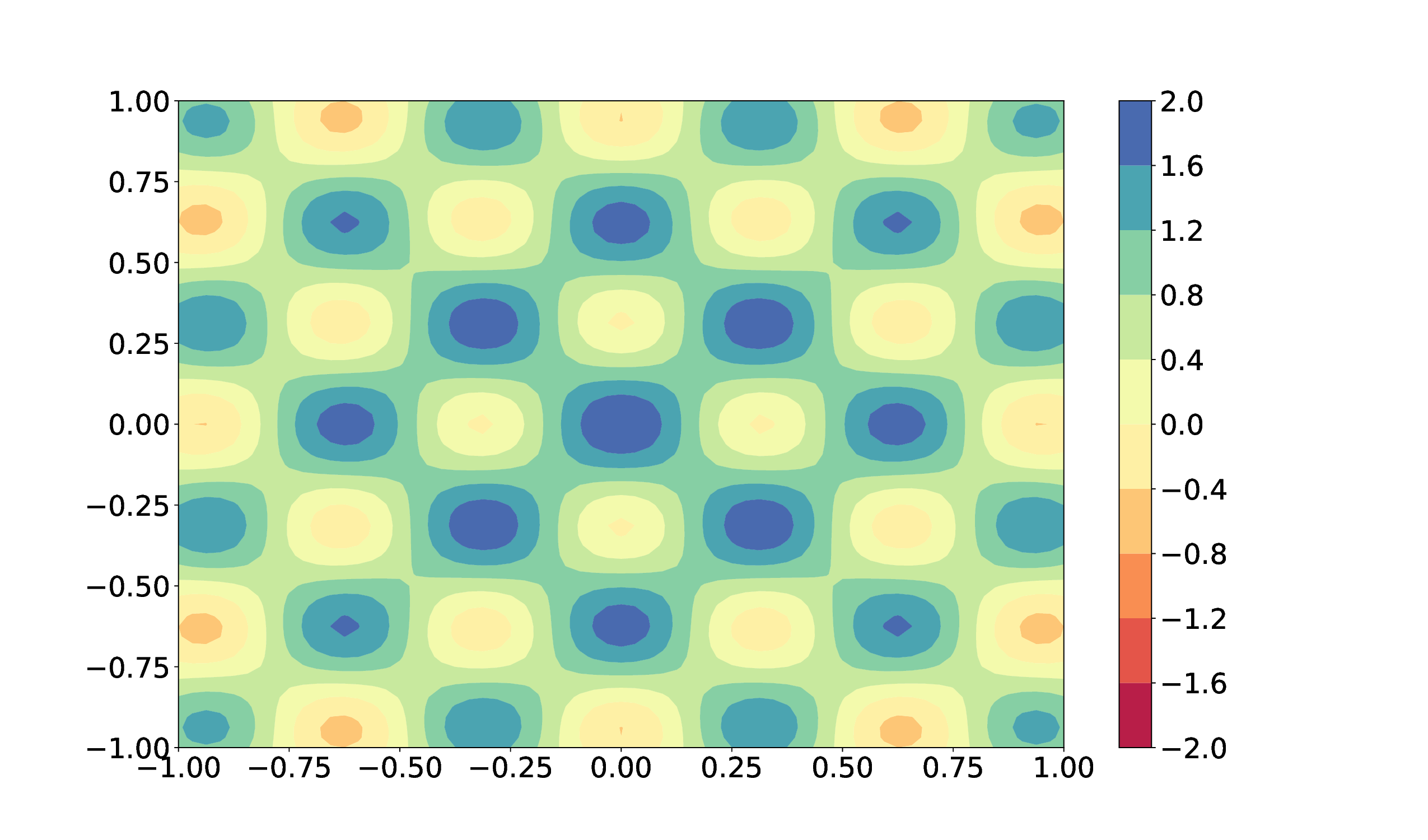}}
	\subfloat[solution of RFM]{\includegraphics[width=0.33\textwidth]{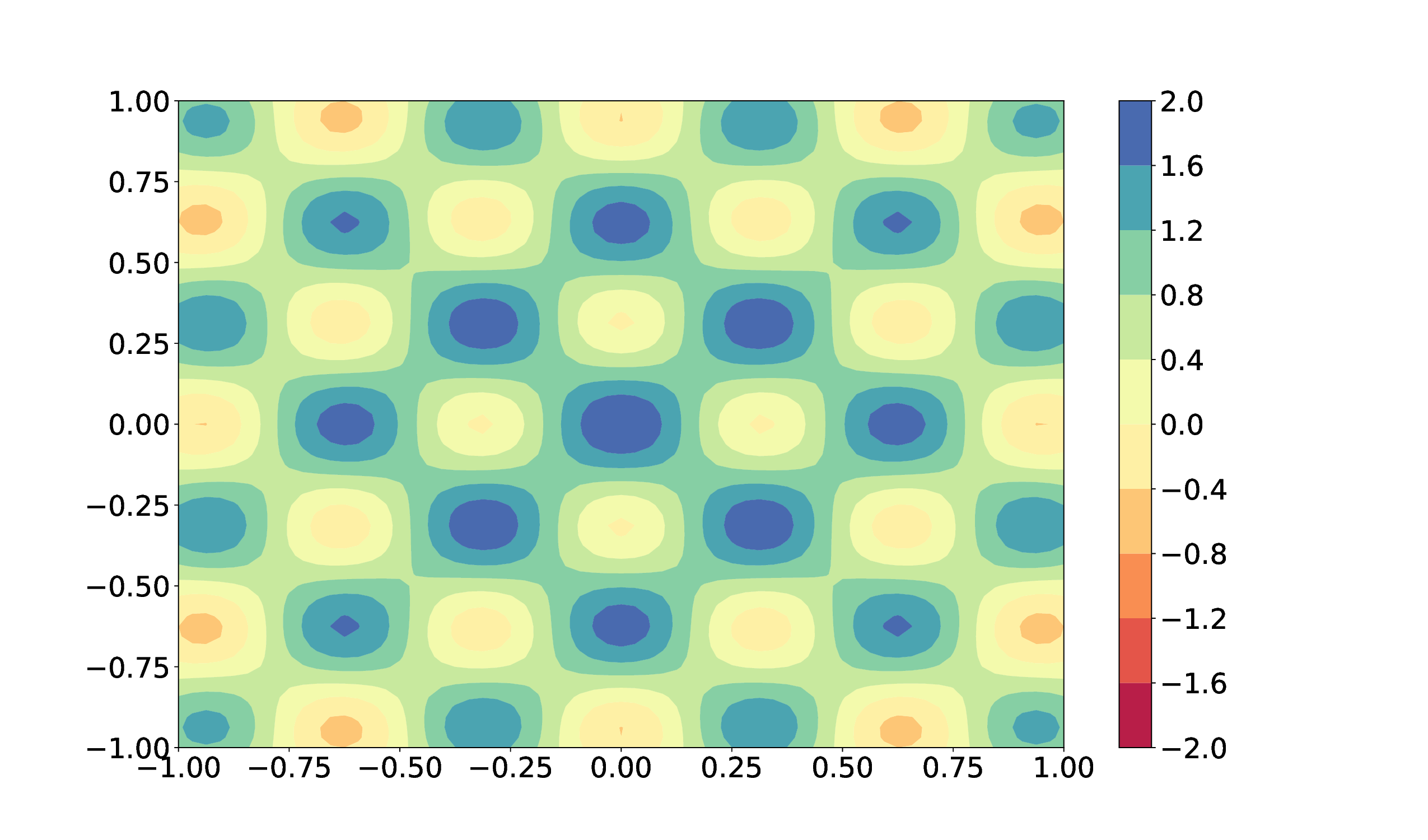}}
	\subfloat[exact solution]{\includegraphics[width=0.33\textwidth]{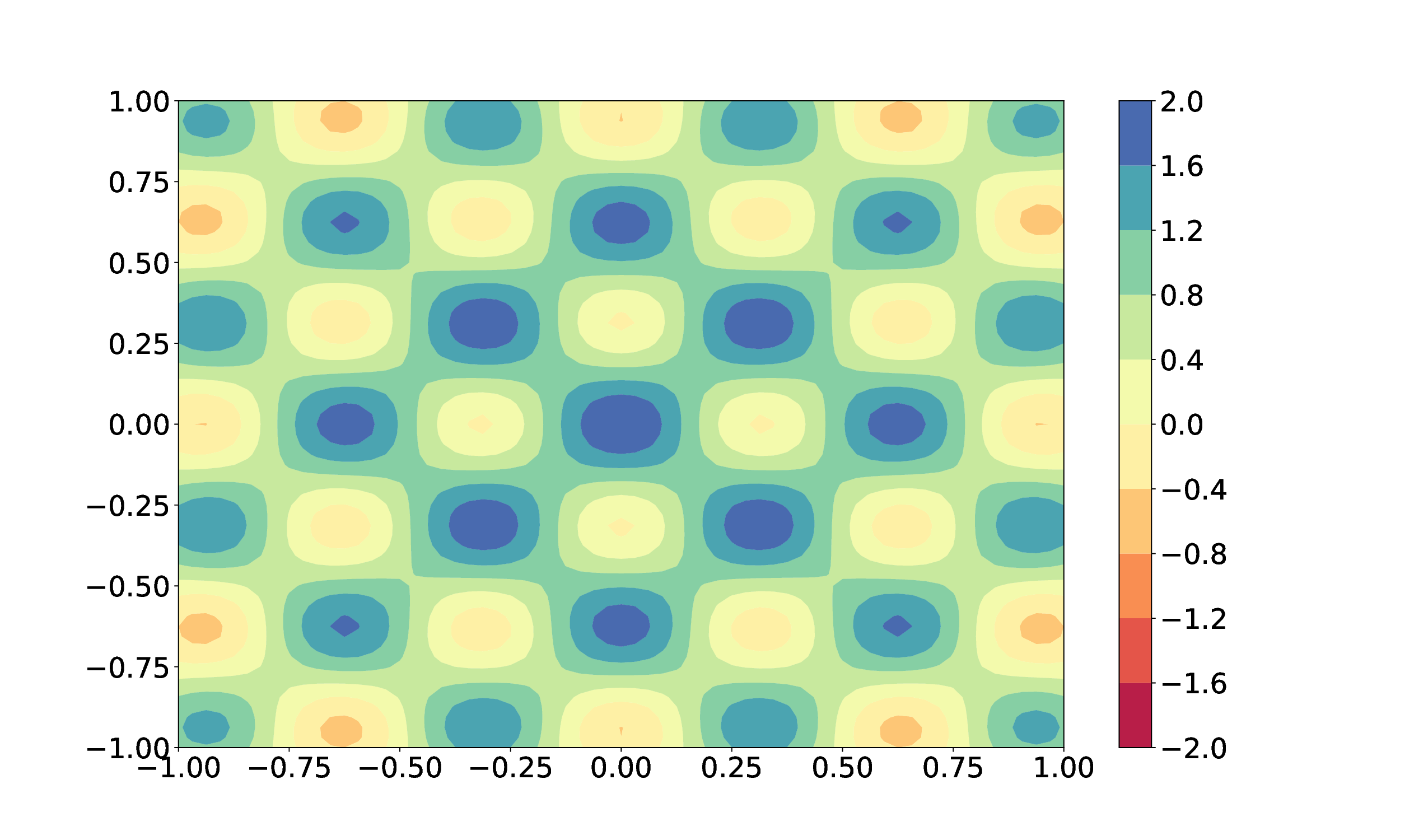}}
	
	\subfloat[error of ELM]{\includegraphics[width=0.33\textwidth]{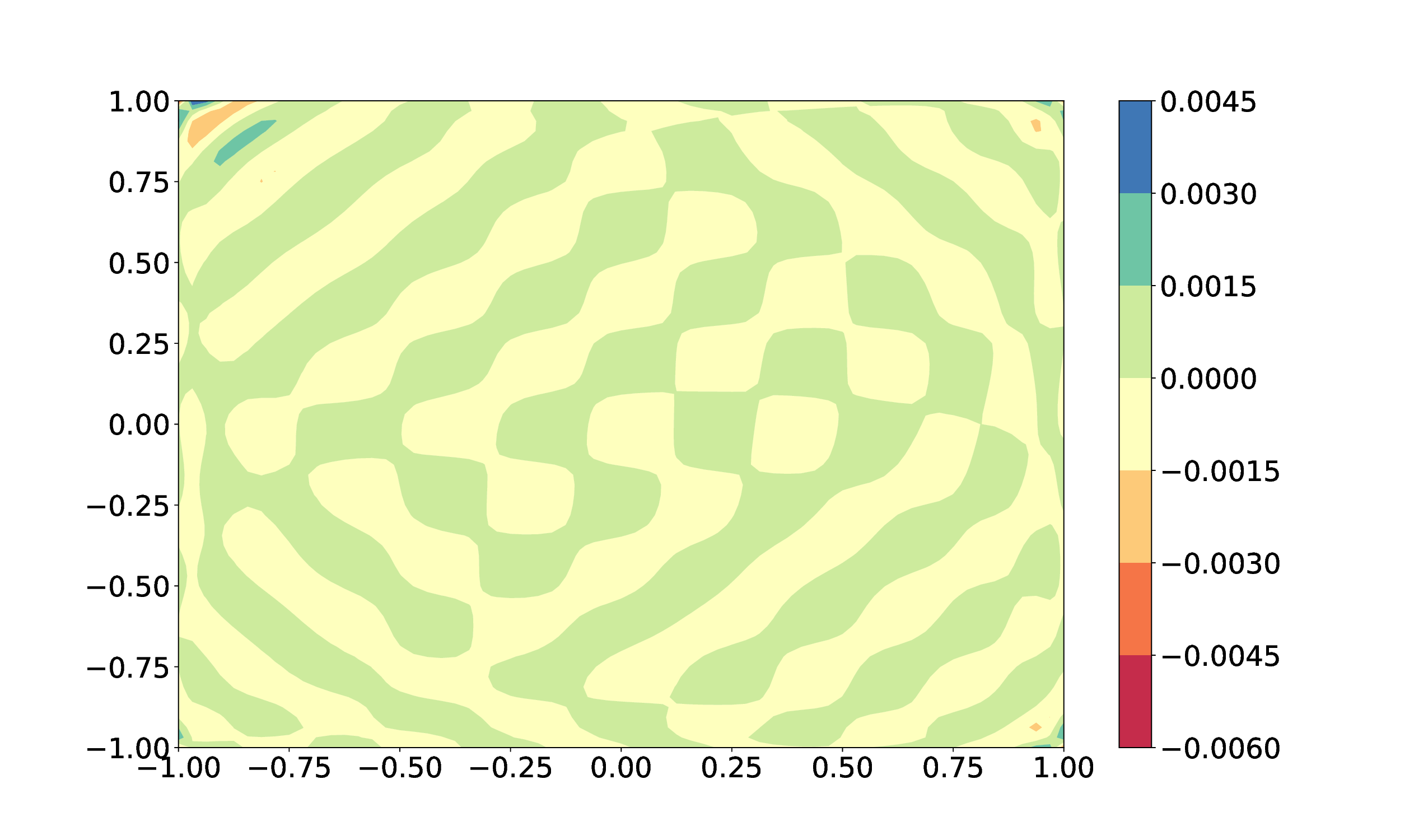}}
	\subfloat[error of RFM]{\includegraphics[width=0.33\textwidth]{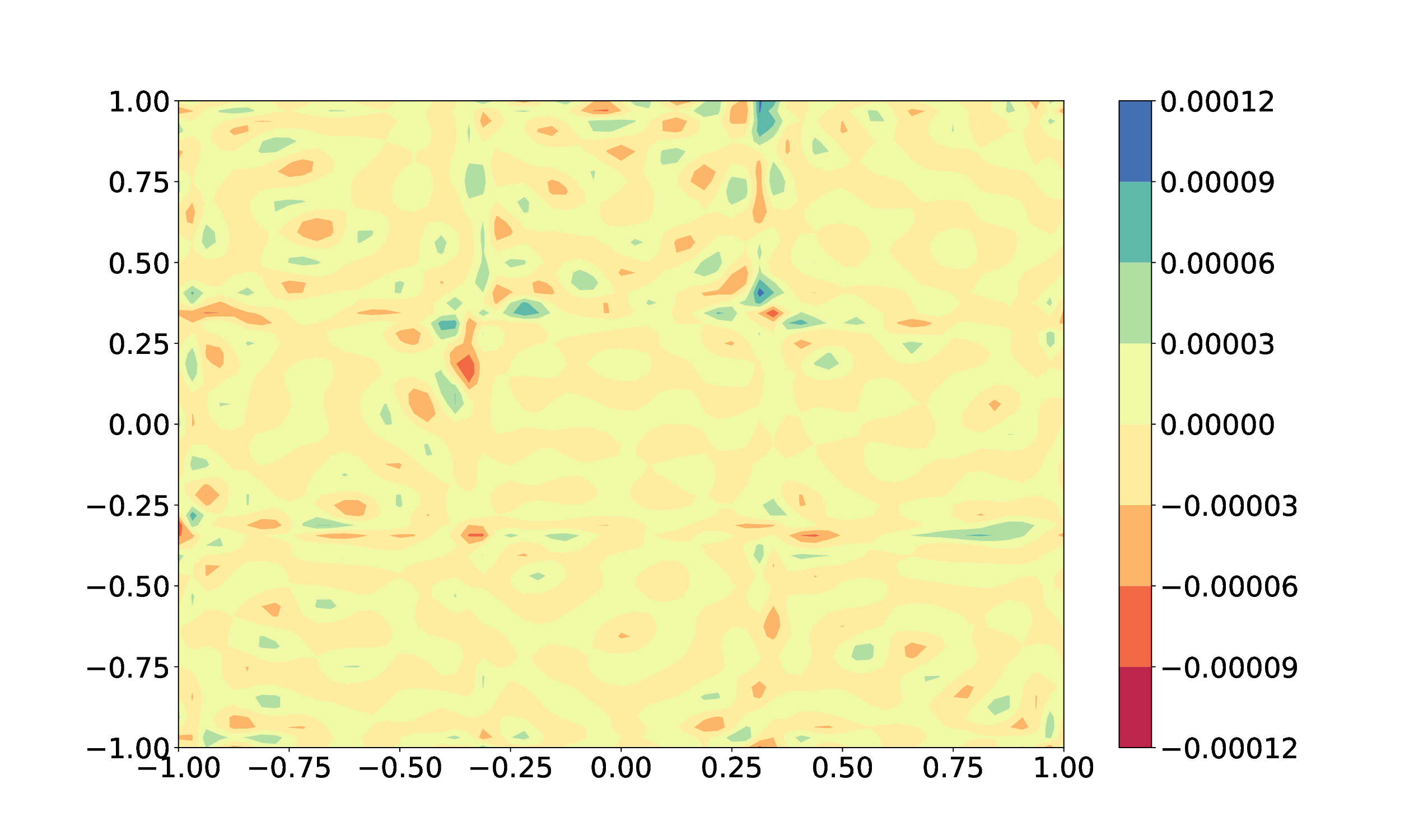}}
	\subfloat[rank and error]{\includegraphics[width=0.33\textwidth]{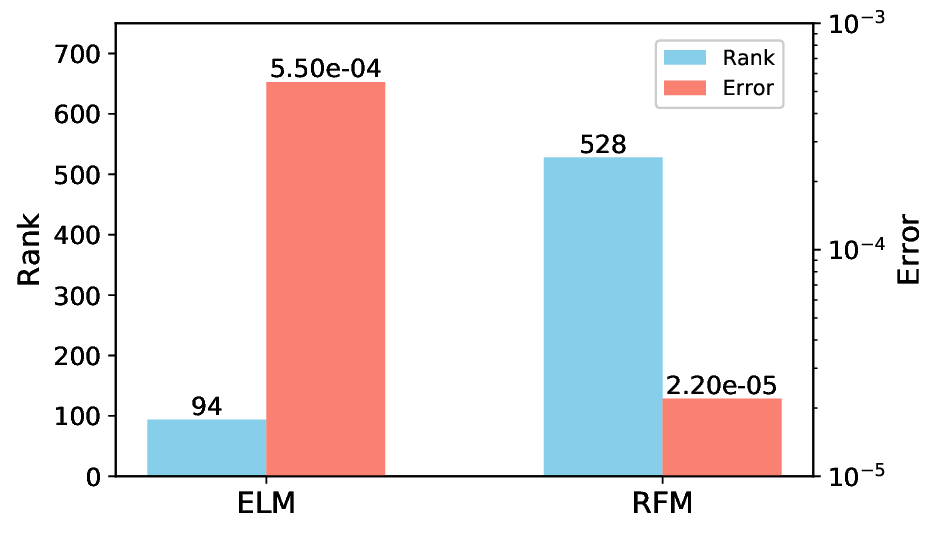}}
	\caption{(\cref{eg::rfm}) (a)-(e) are the solutions and point-wise errors of random feature method and extreme learning machine method. 
		The last subfigure (f) shows the rank and $L^2$ error of two methods.}
	\label{fig::elm_rfm}
\end{figure}

\section{A Pre-training Strategy to Increase $\epsilon$-rank}\label{sec::initialization}
\subsection{Initialization-Induced Rank Deficiency}

The initialization of neural network parameters critically determines the functional diversity of neuron representations. In deep learning, the initial loss plateau-a flat region in the loss landscape during early training—is widely observed. This phenomenon is closely linked to the $\epsilon$-rank deficiency of the initial neuron set, where the functional diversity remains limited to $O(1)$. Standard initialization schemes, such as the Xavier/Glorot method \cite{glorotUnderstandingDifficultyTraining2010}, typically produce neuron functions that approximate linear transformations of their inputs. Consequently, the initial neuron set fails to capture nonlinear target functions, necessitating gradient-based optimization to overcome this $\epsilon$-rank deficiency. This process can be interpreted as the network implicitly performing feature extraction during training.

To formalize this, consider a neuron $\phi(x) = \tanh(wx + b)$ with input $x \in [-1,1]$. The Taylor expansion around $x=0$ yields:
$$
\phi(x) = wx + b - \frac{(wx + b)^3}{3} + o((w+b)^3).
$$
Under Xavier initialization, where $w,b \sim U(-1/\sqrt{n}, 1/\sqrt{n})$, the linear term $wx + b$ dominates due to small weight magnitudes. Consequently, the neuron function $\phi(x) \sim wx + b$ approximates a linear map, and the $\epsilon$-rank of the neuron set $\{\phi_i\}_{i=1}^n$ remains approximately 2 post-initialization. This low rank directly causes the initial loss plateau, as the network lacks sufficient functional diversity to approximate nonlinear target functions.

Theoretical remedies exist to increase the $\epsilon$-rank by expanding the weight distribution (e.g., $w,b \sim U(-n,n)$), but practical implementations avoid this approach due to two critical issues: (1) gradient explosion caused by large weight magnitudes, and (2) activation saturation from $\sigma(wx+b)$ entering the saturation regions of the activation function (e.g., $\sigma(x)\to \pm 1$ for tanh), which results in vanishing gradients and loss of representational flexibility.
These phenomena, documented in \cite{glorotUnderstandingDifficultyTraining2010}, highlight the tension between functional richness and numerical stability in initialization design.

\subsection{$\epsilon$-rank Motivated Pre-training Strategy}

Pre-training in neural networks often refers to the process of initializing parameters to acquire broadly applicable features that facilitate subsequent training. While traditional initialization schemes (e.g., Xavier/Glorot initialization \cite{glorotUnderstandingDifficultyTraining2010}) aim to balance signal propagation, they often fail to explicitly address functional diversity: the capacity of neuron functions to span a rich representational space. This limitation is particularly evident in the early stages of training, where low $\epsilon$-rank restricts the network's ability to approximate nonlinear targets. To overcome this, we propose an $\epsilon$-rank motivated pre-training strategy that directly enhances functional diversity at initialization.

The following numerical experiments (\cref{fig::depth}) reveal a depth-dependent growth in $\epsilon$-rank: within the same deep neural network, the $\epsilon$-rank of the subsequent layer is higher than that of the previous layer. For a network with $L=4$ hidden layers and $n=50$ neurons per layer, we observe the $\epsilon$-rank increases monotonically with depth. The first hidden layer fundamentally differs from deeper layers, as its neuron functions are direct compositions of the activation function without intermediate transformations.

\begin{figure}[htbp]
	\centering
	\includegraphics[width=0.7\textwidth]{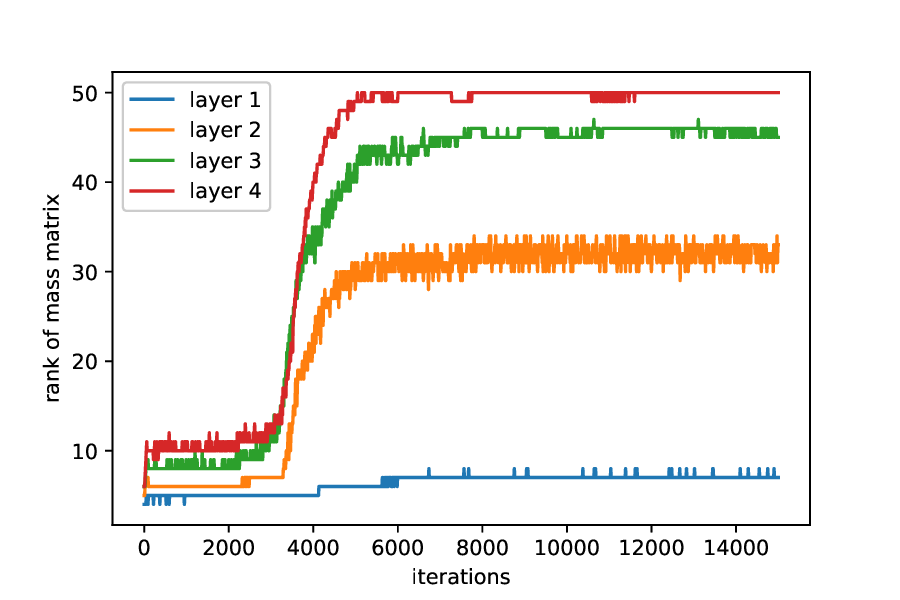}
	\caption{(\cref{eg::ff}) Layer-wise $\epsilon$-rank over training process}
	\label{fig::depth}
\end{figure}

Building on these observations, we propose a targeted pre-training strategy to enhance the $\epsilon$-rank of the first hidden layer. Inspired by basis functions in classical finite element methods, for $x\in[-1,1]$, we define:
\begin{equation}
(y_1)_j = \tanh\left(\frac{n}{2}(x - x_j)\right), \quad x_j = -1 + \frac{2(j-1)}{n}, \quad j=1,\ldots,n.
\label{eq::ut}
\end{equation}

This design ensures three key properties:
\begin{itemize}
    \item \textbf{Localized sensitivity}: Each neuron exhibits significant response within a narrow region of width $\sim 2/n$ around $x_j$, while its influence decays rapidly outside this interval. 
    \item \textbf{Domain-wide representation}: The uniform sampling of $x_j$ ensures that the set $\{x_j\}_{j=1}^n$ forms a dense grid over $[-1,1]$. This guarantees that any input $x \in [-1,1]$ lies within $O(1/n)$ of at least one center $x_j$, enabling localized approximation across the entire domain.
    \item \textbf{Linear independence}: The distinct centers $x_j$ ensure that the set of neurons is $\epsilon$-linearly independent in $L^2([-1,1])$. This independence guarantees an $\epsilon$-rank of $O(n)$ at initialization.
\end{itemize}

While effective in 1D settings, this strategy faces challenges in high dimensions due to the exponential growth of required samples for explicit grid construction.
To address this limitation, we draw inspiration from the random feature method proposed in \cite{zhangTransferableNeuralNetworks2024}, which avoids explicit grid construction by randomly sampling directions $a_j$ on the unit sphere $S^d$. This approach ensures uniform hyperplane density across the input space without requiring exponential sample growth. Building on this idea, we propose a uniform distribution initialization (UDI) pre-training strategy for deep networks, where each neuron in the first hidden layer is parameterized as:
\begin{equation}
	(y_1)_j = \tanh(\gamma(a_j\cdot x+b_j)),\quad j = 1,\cdots, n,
	\label{eq::udi}
\end{equation}
where $\gamma$ is the shape hyperparameter, $\{a_j\}_{j=1}^n$ are directions uniformly sampled on the unit sphere, i.e., 
\begin{equation*}
	a_j \in S^d:= \{a\in\mathbb{R}^d|\|a\|_2=1\},
\end{equation*} 
and $\{b_j\}_{j=1}^n$ are the distinct nodes uniformly sampled from the interval $[0,R]$ with $R$ determined by the domain $\Omega$.
This design bypasses the need for explicit grid construction while maintaining domain-wide representation through directional diversity.

By modifying only the parameters of the first hidden layer, our method avoids the challenges associated with reconfiguring deeper network architectures while preserving parameter efficiency. This approach ensures an $\epsilon$-rank of $O(n)$ from initialization, thereby achieving high functional diversity without increasing the total number of parameters. Numerical experiments demonstrate that this strategy effectively mitigates the initial loss plateau while maintaining numerical stability throughout training. These results align with the core principle of pre-training: to initialize the network with a rich set of diverse functions that accelerate convergence and generalize across tasks.

\subsection{Numerical Experiments}
\label{sec::numerical_pde}
This section systematically evaluates the effectiveness of our proposed pre-training strategy through three canonical experiments: one-dimensional function approximation, two-dimensional function approximation, and solving partial differential equations.

\begin{example}[One-dimensional function fitting, \cref{fig::initial}]\label{eg::ff_initial}
	The target function is also defined as \cref{eg::target}:
\begin{equation*}
	u(x) = \cos(x)+\cos(2x)+\cos(30x),\qquad x\in[-1,1].
\end{equation*}
The strategies are: (1) Deterministic initialization: the first layer is initialized as  
	$$
	(y_1)_j = \tanh\left(\frac{n}{2}(x - x_j)\right), \quad j = 1, \ldots, n,
	$$  
	where $\{x_j\}$ are uniform grid points on $[-1, 1]$; (2) Random Initialization: standard Xavier initialization.
The networks are chosen as a shallow one ($L=2,n=30$) and a deep one ($L=4, n=50$).
\end{example}
\begin{figure}[htbp]
	\centering
	\subfloat[$L=2,n=30$]{\includegraphics[width=0.8\textwidth]{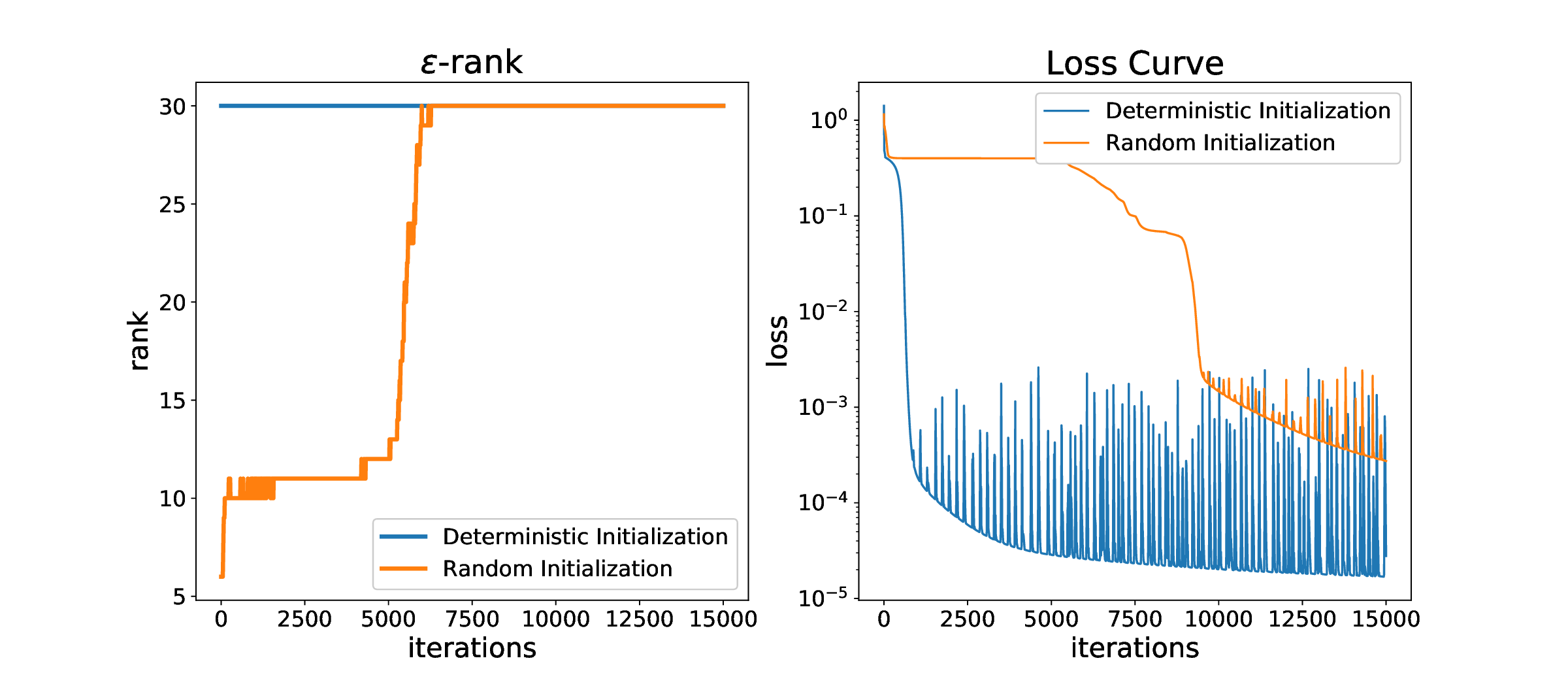}}

	\subfloat[$L=4,n=50$]{\includegraphics[width=0.8\textwidth]{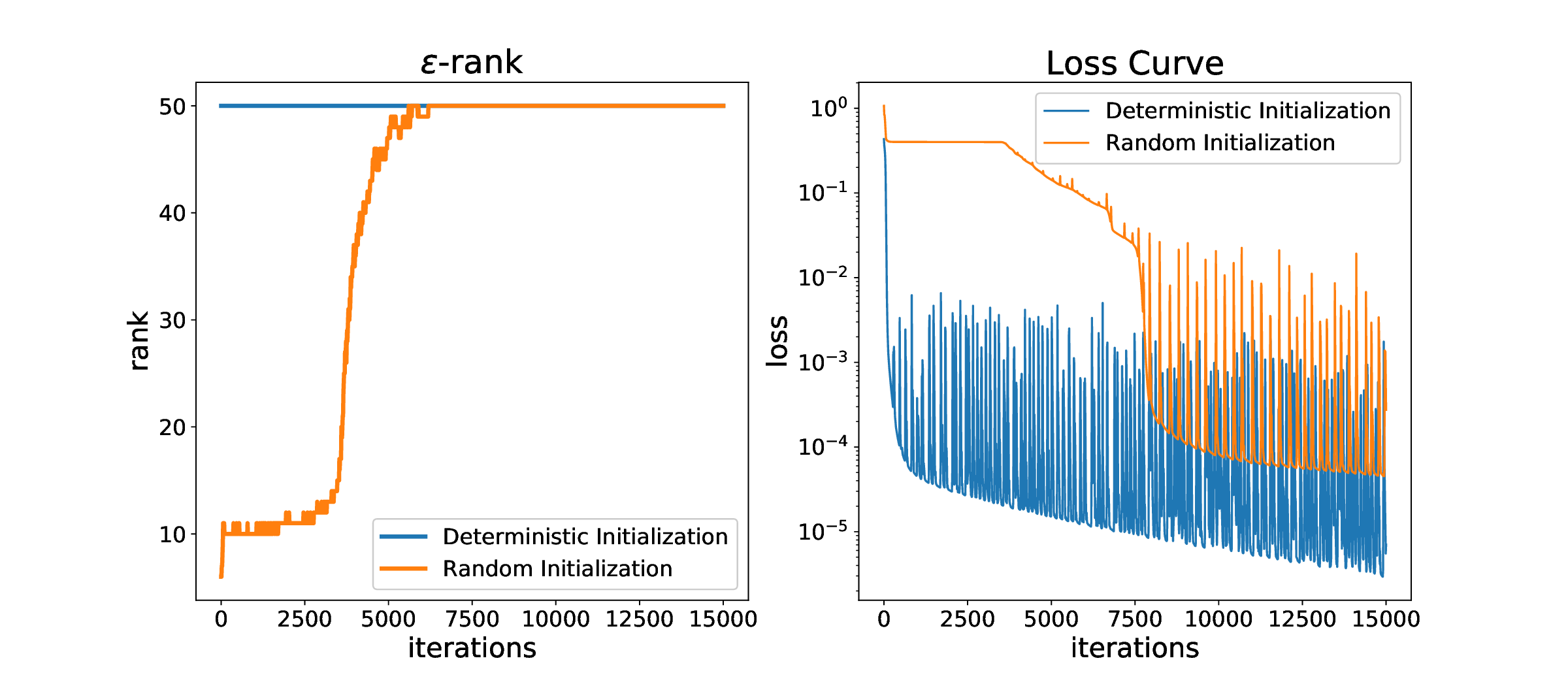}}
	\caption{(\cref{eg::ff_initial}) Training dynamics for one-dimensional function fitting under different initialization and network settings. 
		The left figure shows the evolution of $\epsilon$-rank and the right figure plots the losses. 
		The blue curve gives the result of deterministic initialization method and the orange curve is under default settings.}
	\label{fig::initial}
\end{figure}
\cref{fig::initial} demonstrates the superiority of proposed pre-training strategy.
The $\epsilon$-rank achieves $O(n)$ immediately, whereas random initialization remains low for over 5000 iterations.
This directly correlates with the loss reduction: deterministic initialization reaches $10^{-4}$ in 1000 iterations,
compared to the random initialization after 10,000 iterations.
Notably, comparing the blue curve in \cref{fig::initial}(a)  and the orange curve in \cref{fig::initial}(b), 
even a shallow and narrow network performs better than a deeper and wider network. These findings validate that principled initialization design can simultaneously reduce training time and enhance final accuracy.

\begin{example}[Two-Dimensional Function Fitting, \cref{fig::ladder_fun2d}]\label{eg::ladder_fun2d}
\label{eg::ff_2d}
We extend the experiment to the 2D function:
$$
u(x,y) = e^{-(x^2+y^2)}\sin(5x+5y),\quad (x,y)\in[-1,1]^2,
$$
using a network with $L=3$ hidden layers and $n=50$ neurons per layer. The initialization method for the first layer follows the uniform density initialization (UDI) described in \ref{eq::udi}.
\end{example}

The numerical results demonstrate that the pre-training strategy enforcing linear separation at the first hidden layer significantly accelerates convergence and improves accuracy. 
As shown in \cref{fig::ladder_fun2d}, the proposed method achieves a relative error reduction of over 90\% compared to the baseline approach within 1000 training iterations.
This improvement is directly correlated with the increase in the $\epsilon$-rank of the neuron functions.

\begin{figure}[htbp]
	\centering

\includegraphics[width=\textwidth]{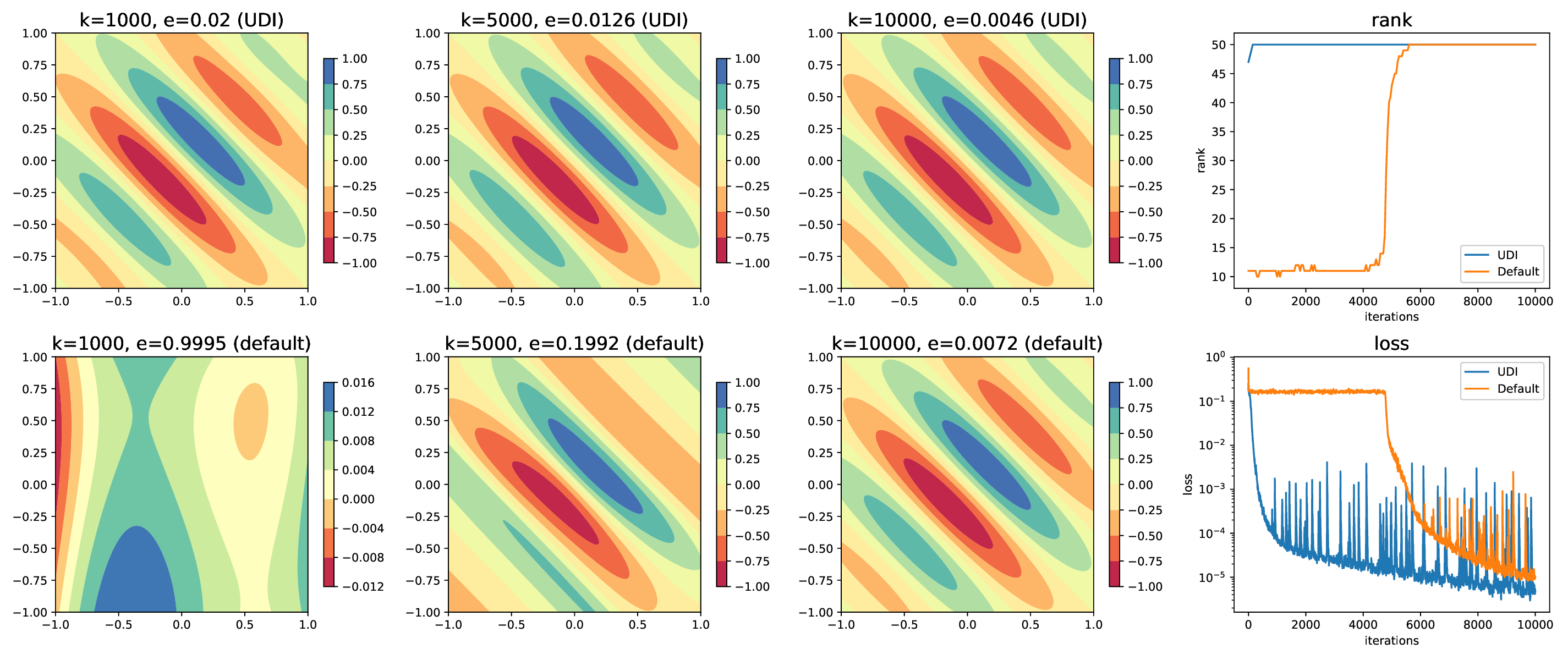}
	\caption{(\cref{eg::ff_2d}) 
	Training dynamics of 2-D function approximation using Xavier initialization (bottom row) and uniform density initialization (UDI, top row). From left to right, the first three columns show contour plots of the predicted solutions at 1000, 5000, and 10,000 training iterations, with relative errors marked.
	The final column displays the logarithmic loss curve and $\epsilon$-rank evolution.}
	\label{fig::ladder_fun2d}
\end{figure}

The performance advantage of the proposed pre-training strategy is observed across varying problem scales.
To further validate its effectiveness, we consider the following boundary value problem:
\begin{example}[Solving Possion's equation by PINN, \cref{fig::pde_ini}] \label{eg::pde}
	\begin{equation*}
		\left\{
		\begin{aligned}
			-\Delta u & = f, && x\in\Omega,\\
			u & = 0, && x\in\partial\Omega,
		\end{aligned}\right.
	\end{equation*}
	where $\displaystyle \Omega = \left[-\frac{\pi}{2},\frac{\pi}{2}\right]^2,\ f(x,y) = 32\sin4x\sin4y$. The analytic solution of this differential equation is $u(x,y)= \sin4x\sin4y$.
	The loss function is defined as 
	\begin{equation*}
		\mathcal{L}(u) = \|\Delta u + f\|_{L^2(\Omega)}^2 + \mu_{bc} \|u\|_{L^2(\partial\Omega)}^2.
	\end{equation*}
\end{example}

The experimental results in \cref{fig::pde_ini} demonstrate that the staircase phenomenon manifests in PINN-based PDE solving.
Under default initialization, the loss decreases slower, while in contrast, the UDI method achieves a faster and smaller loss.
This empirical evidence confirms the direct correlation between feature diversity and training efficiency in both function approximation and PDE solving.

\begin{figure}[htbp]
	\centering
	\subfloat[$n=50$]{\includegraphics[width=0.7\textwidth]{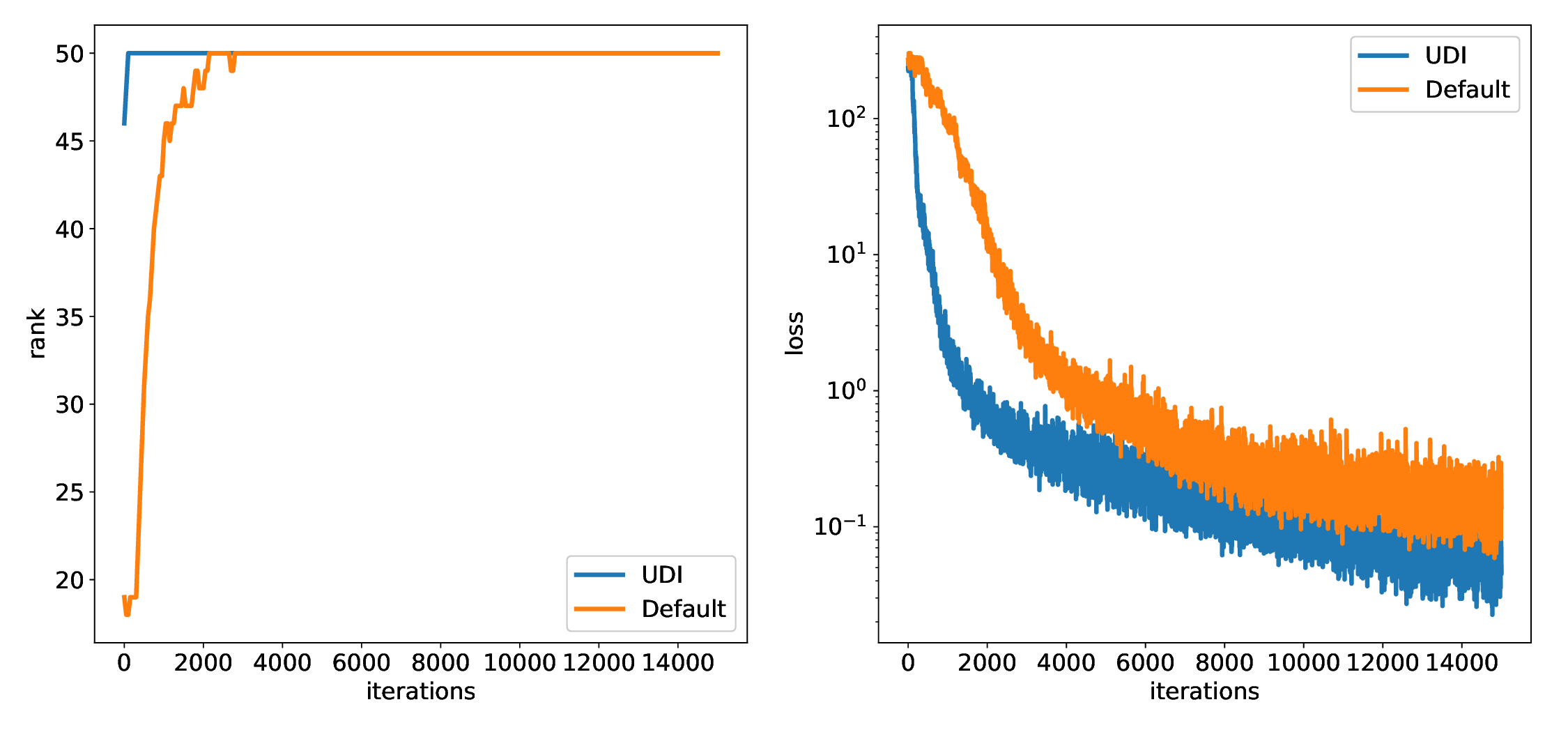}}

	\subfloat[$n=100$]{\includegraphics[width=0.7\textwidth]{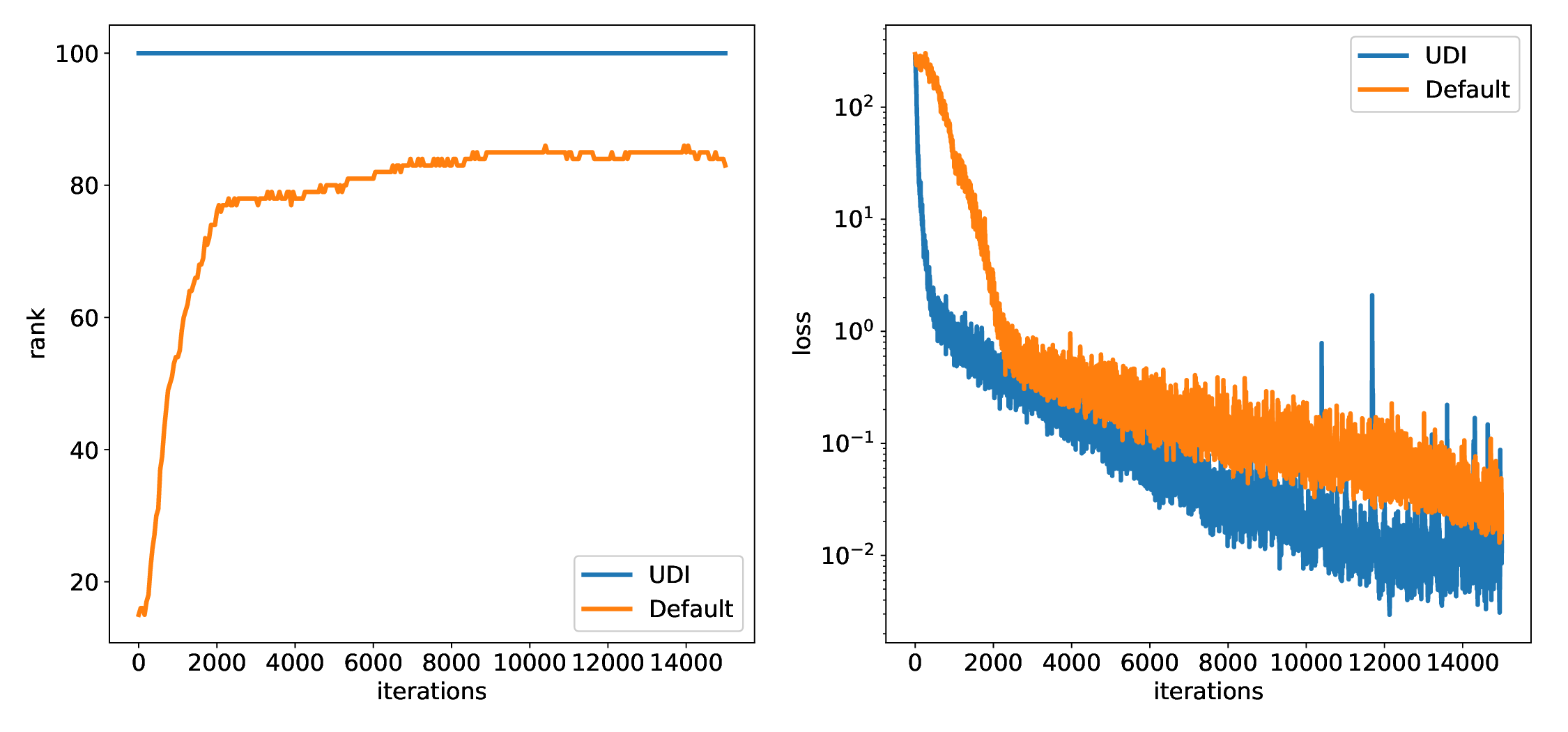}}
	\caption{(\cref{eg::pde})Training dynamics of PINN for the 2D Poisson equation with Xavier initialization (orange) and UDI (blue) under different network width $n=50$ and $n=100$.
		The left figure shows the $\epsilon$-rank and the right figure shows losses.}
	\label{fig::pde_ini}
\end{figure}

All numerical experiments demonstrate that a high $\epsilon$-rank method is critical for both function approximation and PDE solving, particularly in overcoming the early-training plateau.
These results validate that principled pre-training design can simultaneously reduce training time and enhance final accuracy without increasing model complexity or parameter count.

\section{Concluding Remarks}

In summary, this research provides a novel perspective on the training dynamics of 
deep neural networks by drawing connections to traditional numerical analysis. 
A key finding of our study is the identification of the \textit{staircase phenomenon}, which describes a stepwise increase in the linear independence of neuron functions during the training process, typically associated with rapid decreases in the loss function.
This finding highlights the importance of establishing a diverse and robust set of neuron functions in the early stages of training, which promotes both efficiency and convergence in model optimization.

Theoretical analyses and numerical experiments confirm that a set of linearly independent basis neuron functions is essential for effectively minimizing the loss function of neural networks.
The training process can be significantly accelerated by leveraging appropriate techniques, such as $\epsilon$-rank motivated pre-training strategy, constructing efficient network architectures, and employing proper domain partitioning.
These strategies have shown effective in promoting the linear independence of neuron functions, providing insight into the success of certain neural network approaches.

This study bridges the gap between neural network frameworks and traditional numerical methods, offering deeper insights into the mechanisms underlying deep learning.
It provides a solid foundation for future innovations in network initialization, training methodologies, and the design of more efficient models.

\section*{Acknowledgements}
This work is supported by the National Science Foundation of China (No.12271240, 12426312), the fund of the Guangdong Provincial Key Laboratory of Computational Science and Material Design, China (No.2019B030301001), and the Shenzhen Natural Science Fund (RCJC20210609103819018).

\appendix
\section{Numerical Settings}
This section summarizes the numerical settings for all numerical examples.

The neural network structure is defined as \cref{st::NN} with $\tanh$ activation function.
$L^2$ norm is used to measure errors between predicted solution $u$ and exact solution or reference solution $u^*$. 
\begin{equation*}
	e =  \|u  - u^*\|, \qquad \tilde{e} = \frac{\|u-u^*\|}{\|u^*\|},
\end{equation*}
where $\displaystyle \|u\| = \left(\int_\Omega u^2\mathrm{d}x\right)^\frac{1}{2}$ is the short note of $L^2$ norm.

In numerical calculations, we extend the $\epsilon$-linear independence to the discrete form.
Given $n$ functions discretized on $m$ nodes 
\begin{equation*}
	D =\begin{pmatrix}
		\phi_1(x_1) & \cdots & \phi_n(x_1)\\
		\phi_1(x_2) & \cdots & \phi_n(x_2)\\
		\vdots & \ddots & \vdots\\
		\phi_1(x_m) &\cdots &\phi_n(x_m)
	\end{pmatrix},
\end{equation*}
the Gram matrix is computed by $M = D^TWD$, where $W$ is the weight matrix.
The $\epsilon$-linear independence of these $n$ functions is 
\begin{equation*}
	r_\epsilon(M) = \#\{|\lambda(M)|>\epsilon\}.
\end{equation*}
where $|A|$ is the cardinal number of $A$, and the tolerance is given $\epsilon = 10^{-6}$ unless otherwise mentioned in \cref{tb::settings}.

In low-dimensional cases, $M$ is approximated using numerical integration, while in high dimensions, $M$ is approximated by Monte Carlo integration.
For simplicity, we generate one set of integration data points $\{x\}_{k=1}^m$, which can be Gaussian integral points or uniform mesh for $d=1,2$. Then 
\begin{equation}
	(M)_{ij}\approx  \sum_{k=1}^m w_k\phi_i(x_k)\phi_j(x_k), 
\end{equation}
where $w_k$ are the integral weights.
For example, $w_0=w_m = \frac{1}{2}$ and $w_i = 1, i=2,\cdots,m-1$ is the trapezoidal formulation.

The notations are listed in \cref{tb::notations}.
\begin{table}[htbp]
	\centering
	\caption{The List of Notations}
	\begin{tabular}{|c|l|}
		\hline
		Notation & Stands for ... \\
		\hline
		$L$ & Depth of hidden layers \\
		\hline
		$n$ & Width of hidden layers\\
		\hline
		$m$ & Sample size to calculate the Gram matrix\\
		\hline
		$d$ & Dimension of the input\\
		\hline
		$\sigma$ & Activation functions \\
		\hline
		$\phi$ & The neurons of the output layer\\
		\hline
		$\beta$ & Coefficients of the output layer\\
		\hline
		$M(\phi)$ & The Gram matrix of neuron functions $\{\phi_j\}_{j=1}^n$\\
		\hline
		$r_\epsilon(M)$ & The $\epsilon$-rank of the Gram matrix\\
		\hline
		$\epsilon$ & Tolerance of eigenvalues\\
		\hline
		$N$ & Sample size of training \\
		\hline
		$\mu$ & Weights for losses in PINNs\\
		\hline
		$\gamma$ & Shape parameter in UDI\\
		\hline
	\end{tabular}
	\label{tb::notations}
\end{table}

The settings for numerical examples are listed in \cref{tb::settings} and \cref{tb::hr}.
Activation functions ReLU, ELU, cosine and sigmoid are used in \cref{eg::ff}(b), the sigmoid activation function is 
used in \cref{eg::hr}, and tanh is used for all other examples. Except for the handwriting recognition task, the width of the hidden layers is fixed within each network in all other examples, meaning that every hidden layer in a given network contains the same number of neurons.

\begin{table}[htbp]
    \centering
    \caption{Numerical experiment settings in function fitting and solving PDEs}
    \begin{tabular}{@{}lcccccc@{}}
        \toprule
	 \textbf{Example} & $L$ & $n$ & $m$ & $N$ & $\epsilon$ & \textbf{Other Parameters} \\
        \midrule
		\multicolumn{7}{c}{\textbf{Function Fitting}}
		\\
        
        \midrule
        \cref{eg::ff}(a)       & 2/4      & 25/50     & 129         & 250           & $10^{-6}$        & - \\
        \cref{eg::ff}(b)       & 3      & 50     & 129         & 250           & $10^{-6}$        & varying $\sigma$  \\
        \cref{eg::ff}(c)       & 4      & 50     & 129         & 250           & $10^{-6}$        & - \\
        \cref{eg::ff_initial}       & 2/4      & 30/50     & 129         & 250           & $10^{-6}$        & - \\
        \cref{eg::ff_2d}       & 3        & 50        & $129^2$        & 550           & $10^{-6}$        & $\gamma=2$ \\
        \midrule
        \multicolumn{7}{c}{\textbf{PDE Solving}} \\
        \midrule
        \cref{eg::ac}        & 3        & 50         & $100^2$         & 250      & $10^{-8}$        & $\mu_{\text{bc}} = \mu_{\text{ic}}=1$ \\
        \cref{eg::heat}      & 3        & 100       & $50^3$         & 2050/12550  & $10^{-6}$        & $\mu_{\text{bc}} = \mu_{\text{ic}}=1$ \\
		\cref{eg::rfm}  & 1 & 900 & $65^2$ & $63^2$ & $10^{-12}$  & $3\times 3$ subdomains\\
        \cref{eg::pde}       & 2       & 50/100    & $129^2$        & 250     & $10^{-6}$        & $\gamma=n/50$, $\mu_{\text{bc}}=20$ \\
        \bottomrule
    \end{tabular}
	\label{tb::settings}
\end{table}
\begin{table}[htbp]
    \centering
    \caption{Numerical experiment settings in handwriting recognition}
    \begin{tabular}{@{}lccccc@{}}
        \toprule
		\multicolumn{6}{c}{\textbf{Handwriting Recognition}}\\
	
        \midrule
		 \textbf{Example} & $L$ & $\epsilon$ & epochs & batchsize & network width  \\
        
        \midrule
        \cref{eg::hr}      & 3  & $10^{-6}$     & 3         & 40           & $784\times 10\times 100\times 100\times 10$         \\
        \bottomrule
    \end{tabular}
	\label{tb::hr}
\end{table}

\bibliographystyle{plain}
\bibliography{ref}
\end{document}